\documentclass[a4paper,fleqn]{cas-sc}

\usepackage[numbers]{natbib}
\def\tsc#1{\csdef{#1}{\textsc{\lowercase{#1}}\xspace}}
\tsc{WGM}
\tsc{QE}
\tsc{EP}
\tsc{PMS}
\tsc{BEC}
\tsc{DE}
\newtheorem{theorem}{\bf Theorem}[section]
\newtheorem{remark}{\bf Remark}[section]

\usepackage[utf8]{inputenc} % allow utf-8 input
\usepackage[T1]{fontenc}    % use 8-bit T1 fonts
\usepackage{booktabs}       % professional-quality tables
\usepackage{amsfonts}       % blackboard math symbols
\usepackage{nicefrac}       % compact symbols for 1/2, etc.
\usepackage{microtype}      % microtypography
\usepackage{graphicx}
\usepackage{algorithm}
\usepackage{subcaption}
\usepackage{makecell}
\usepackage{amssymb}
\usepackage{amsmath}
\usepackage{color}
\usepackage{appendix}
\usepackage{tikz}
\usepackage[norelsize,boxed,linesnumbered,vlined,algo2e]{algorithm2e}
\usetikzlibrary{positioning}% To get more advances positioning options
\usetikzlibrary{arrows}% To get more arrow heads
\usepackage{bm}
\usepackage{multirow}

\newenvironment{proof}{{\noindent\it Proof}\quad}{\hfill $\square$\par}

\begin{document}
\let\WriteBookmarks\relax
\def\floatpagepagefraction{1}
\def\textpagefraction{.001}
\shorttitle{linear product initialization}
\shortauthors{Q. Chen et~al.}
%\begin{frontmatter}

\title [mode = title]{A weight initialization based on the linear product structure for neural networks}

\author[1]{Qipin Chen}

\address[1]{Department of Mathematics, Pennsylvania State University, University Park, PA 16802}
\author[1]{Wenrui Hao}\cormark[1]

\author[2]{Juncai He}

\address[2]{Department of Mathematics, The University of Texas at Austin, Austin, TX 78712}

\begin{abstract}
Weight initialization plays an important role in training neural networks and also affects tremendous deep learning applications.  Various weight initialization strategies have already been developed for different activation functions with different neural networks. These initialization algorithms are based on minimizing the variance of the parameters between layers and might still fail when neural networks are deep, e.g., dying ReLU. To address this challenge, we study neural networks from a nonlinear computation point of view and propose a novel weight initialization strategy that
is based on the linear product structure (LPS) of neural networks. The proposed strategy is derived from the polynomial approximation of activation functions by using theories of numerical algebraic geometry to guarantee to find all the local minima. We also provide a theoretical analysis that the LPS initialization has a lower probability of dying ReLU comparing to other existing initialization strategies. Finally, we test the LPS initialization algorithm on both fully connected neural networks and convolutional neural networks to show its feasibility,  efficiency, and robustness on public datasets.
\end{abstract}

\begin{highlights}
\item We consider the neural network training from a nonlinear computation point of view;
\item A new linear product structure initialization strategy has been developed for training neural networks;
\item Theoretical analysis shows that the LPS initialization yields a low probability of dying ReLU.
\end{highlights}

\begin{keywords}
Weight initialization\sep linear product structure \sep neural networks \sep nonlinear computation\end{keywords}

\maketitle

\section{Introduction}

With the rapid growth of applications of neural networks to large datasets, the initialization of the weights of neural networks affects the training process and accuracy significantly. It is well known that zero initialization or arbitrary random initialization can slow down or
even completely stall the convergence process.
This is the so-called problem of exploding or vanishing gradients which in turn slows down the backpropagation and retards the overall training process \cite{pascanu2013difficulty}. Exploding gradients occur when  the gradients get larger and larger, result in oscillating around the minima or even blow up in the training process; vanishing gradients
are the exact opposite of exploding gradients when the gradient gets smaller and smaller due to the backpropagation, cause the slower convergence, and may even completely stop the training process.
Therefore, proper initialization of the weights in training neural networks is necessary \cite{mishkin2015all,nguyen1990improving}.
 The most popular initialization method is to use samples
from a normal distribution, $\mathcal{N} (0, \sigma^2)$ where $\sigma$ is chosen to ensure that the variance of the
outputs from the different layers is approximately the same.  The first systematic analysis of this initialization was conducted in \cite{glorot2010understanding} which
showed that, for a linear activation function, the optimal value of $\sigma^2=1/d_i$, where $d_i$ is the number
of nodes feeding into that layer. Although this study makes several assumptions about the inputs to
the model, it works extremely well in many cases (especially for $tanh(z)$) and is widely used in the initialization of neural
networks commonly referred to {\it Xavier initialization}.  Another important follow-up work is called {\it He initialization} \cite{he2015delving} which
argues that {\it Xavier initialization} does
not work well with the ReLU activation function and changes $\sigma^2=2/d_i$ to achieve tremendous success in ReLU neural networks such as ResNet.
 Recently the weight initialization has become an active research area, and numerous methods \cite{arpit2019initialize,kumar2017weight,mishkin2015all,pennington2017resurrecting,pennington2018emergence,poole2016exponential,saxe2013exact,sussillo2014random}  have been developed to initialize the weights of different neural networks. All the aforementioned initialization works are based on minimizing the variance of parameters between the deeper layers to avoid vanishing/popping at the beginning of training but do not consider the nonlinearity of neural networks which could let the initialization help the final training performance further.

The main contribution of this paper is to study neural networks from the nonlinear computation point of view \cite{chen2019homotopy,hao2021gradient}. We approximate the activation functions by polynomials to provide a new weight initialization approach. The proposed weight initialization algorithm is based on the linear product structure of neural networks and has a theoretical guarantee to find all the local minima based on theories of numerical algebraic geometry \cite{hao2014numerical,SWbook,mehta2021loss}. Further theoretical analysis reveals that our new initialization method has a low probability of dying ReLU for deep neural networks.   Numerical experiments on both fully connected neural networks and convolutional neural networks show the feasibility and efficiency of the proposed initialization algorithm.

\section{Problem setup and polynomial approximation of activation functions}

By considering a $(n+1)$-layer neural network $y(x;\theta)$, we represent the output, $y$, in terms of the input, $x$, as
	\begin{equation}
		\label{abstract_network}
		y(x;\theta) = W^nf^{n-1}+b^n, ~f^{\ell}=\sigma(W^\ell f^{\ell-1}+b^\ell), \ell\in\{1,\dots, n-1\}, \hbox{and }f^0=x,
	\end{equation}
	where $W^\ell \in R^{m_{\ell}\times m_{\ell-1}}$ is the weight matrix, $b^\ell \in R^{m_{\ell}}$ is the bias vector, $m_\ell$ is the width of the $\ell$-th layer, $m_0=\dim(x)$, $m_{n}=\dim(y)$, and $\sigma$ is the activation function.  For simplicity, we denote the set of all parameters as $\theta =
	\{W^\ell,b^\ell\}_{\ell=1}^n$ and 	the number of all parameters as $|\theta|$.
 The activation function, $\sigma$, is a nonlinear function but not a polynomial since compositions of second-order and higher-order polynomials yield unbounded derivatives and could lead to exploding gradients \cite{chon1997linear,ma2005constructive}. It is well known that nonlinear activation functions play important roles in deep neural networks but the nonlinearities still remain unclear in the mathematical context. In order to quantify these nonlinearities, we apply  the polynomial approximation for  nonlinear activation functions. In particular, we use the Legendre polynomial approximation for activation functions and denote the orthogonal polynomial space
as $\mathbb{P}_d = {\rm span} \{L_k(x),~  0\le k \le d \}$, where $L_k(x)$ is the  Legendre polynomial with degree $k$ \cite{xiu2002wiener}. Then the approximated polynomial of any given activation function {{$\sigma\in L^2([-1,1])$}} is
\begin{equation}\label{eq:projection1}
P_d \sigma := \sum_{k=0}^d \alpha_k L_k(x) \hbox{~and~} \alpha_k= \frac{1}{\|L_k(x)\|^2} (\sigma, L_k).
\end{equation}
If $\sigma$ is the ReLU activation function, we have the following theorem.
\begin{theorem}
If $\sigma(x)=ReLU(x)$, the estimate of the Legendre polynomial approximation on $[-1,1]$ is
\begin{equation}\label{key}
\|\sigma - P_d \sigma \|_{L^2([-1,1])} \le C\frac{1}{d} \|\sigma\|_{H^1([-1,1])},
\end{equation}
where $C$ is a constant and {{
$\displaystyle\|\sigma\|_{H^1([-1,1])}=\sqrt{\int_{-1}^1 \sigma^2(x)+(\sigma'(x))^2dx}.$}} {In particular, $\displaystyle\|ReLU(x)\|_{H^1([-1,1])}=\frac{4}{3}$.}
Moreover, the explicit formula of $\alpha_k$ in (\ref{eq:projection1}) is
	\begin{equation}
	\alpha_k =
	\begin{cases}
	\frac{(-1)^m(2k+1)}{2(2-k(k+1))4^m} \binom{2m}{m} \quad &\text{for} \quad k = 2m,\\
	0 \quad &\text{for}\quad k = 2m+1.
	\end{cases}\nonumber
	\end{equation}
\end{theorem}

Based on the theory of numerical algebraic geometry, the polynomial system using the  $P_d \sigma$ activation function can be solved by the following homotopy setup \cite{SWbook,hao2020adaptive}
	\begin{equation}
		\label{homotopy}
		H(\theta,t)=(1-t) (\tilde{y}(x_i,\theta)-y_i) +tG(\theta)=0, i=1,\dots N,
	\end{equation}
where $\tilde{y}(x,\theta)=W^n \tilde{f}^{n-1}+b^n$, $\tilde{f}^{\ell}=P_d\sigma(W^\ell \tilde{f}^{\ell-1}+b^\ell)$, $x_i$ and $y_i$ are sample points ($N$ is the number of sample points), $t$ is the homotopy parameter, and $G(\theta)$ is a polynomial system with known solutions.  Then solutions of $\tilde{y}(x_i,\theta)=y_i$ can be solved by tracking $t$ from $1$ to $0$ via this homotopy. The start system $G(\theta)$ is formed via closely mirroring the structure of $P_d \sigma$ \cite{SWbook} such as the total degree start system \big($G(\theta)$ has the same degree as $\tilde{y}(x_i,\theta)$\big), the multi-homogeneous start system (dividing the variables into several homogenous groups), the linear product start systems (dividing into several linear systems), and etc \cite{bates2013numerically}. Homotopy continuation in this context is theoretically guaranteed to compute all solutions of $\theta$ due to Bertini's theorem \cite{bates2013numerically,SWbook}. But the number of solutions of  $G(\theta)$ grows exponentially and cannot be computed directly when the neural network becomes wide and deep (See illustrative examples in Appendix \ref{sec:NE}). Therefore, we use this theory to initialize the weights of neural networks instead of solving it directly.

\section{Linear product structure and weight initialization}
After approximated by a polynomial, namely, $\sigma(x)\approx \mathcal{P}_2(x)$, the neural network representation in (\ref{abstract_network}) becomes
	\begin{eqnarray}
		\label{linear_productm1}
		y(x;\theta) \approx W^n \mathcal{P}_2(W^{n-1} \tilde{f}^{n-1}+b^{n-1}) + b^n.
	\end{eqnarray}
For each component $y_j(x;\theta)$, $j=1,\dots, \dim(y)$, we decompose the polynomial expression (\ref{linear_productm1}) into a linear product structure \cite{bates2013numerically,SWbook}, namely,
	\begin{equation}
		y_j(x;\theta) \approx W^n_{j} \mathcal{P}_2(W^{n-1} \tilde{f}^{n-1}+b^{n-1}) + b^n_j\in \{W^n_{j}, b^n_j,1\} \times \{\mathcal{P}_2(W^{n-1} \tilde{f}^{n-1}+b^{n-1}), 1\},
	\end{equation}
where $W^n_{j}$ is the $j$-th row of $W^n$ and  $\{W^n_{j}, b^n_j,1\} $  represents the linear space generated by variables $W^n_{j}$, $b^n_j$, and $1$ (Similar for $\{\mathcal{P}_2(W^{n-1} \tilde{f}^{n-1}+b^{n-1}), 1\}$).  {{More specifically, the linear product structure means that the approximated polynomial of $y_j(x;\theta)$ is a special case of the product of two linear spaces generated by  $\{W^n_{j}, b^n_j,1\}$ and $\{\mathcal{P}_2(W^{n-1} \tilde{f}^{n-1}+b^{n-1}), 1\}$.}}
Then the homotopy setup in (\ref{homotopy}) is revised as
	\begin{eqnarray}
		\label{linear_product1}
		H(W^n_{j},b^n_j,t)&=&(1-t) (W^n_{j} \mathcal{P}_2(W^{n-1} \tilde{f}^{n-1}+b^{n-1}) + b^n_j-y_{j}) \nonumber +t(\alpha_1 W^n_{j}+\alpha_2b^n_j+\alpha_3)(\beta_1\mathcal{P}_2(W^{n-1} \tilde{f}^{n-1}+b^{n-1})+\beta_2),\\
	\end{eqnarray}
where  $\alpha_i$ and $\beta_i$ are random numbers (referred to the generic points in \cite{bates2013numerically,SWbook}) and are assumed to  follow the normalized Gaussian distribution, namely, $\mathcal{N}(0,1)$.
When $t=1$, the start system, $H(W^n_{j},b^n_j,1)=0$ ($G(\theta)=0$ in (\ref{homotopy})), is solved by two linear systems below
	\begin{equation}
A\left(
\begin{array}{c}
\big(W^n_{j}\big)^T\\
b^n_j
\end{array}\right)=-\alpha_3 \hbox{~and~} \beta_1\mathcal{P}_2(x)(\cdot)=-\beta_2,
	\end{equation} where $A=[\alpha_1, \alpha_2]\in R^{(m_{n-1}+1)\times (m_{n-1}+1)}$, $\alpha_3\in R^{(m_{n-1}+1)\times 1}$, $\beta_1\in R^{m_{n-1}\times m_{n-1}}$, and $\beta_2\in R^{m_{n-1}\times 1}$.
 Based on the Xavier initialization \cite{pascanu2013difficulty}, we have the variances between the input and the output be  identical, namely,
	\begin{equation}
(m_{n-1}+1)var(A_{k,i}) var(W^n_{j,i})=var(\alpha_{3,k}),\nonumber
	\end{equation}
  which implies that \[var(W^{n}_{j,i})=var(b^n_{j})=\frac{1}{m_{n-1}+1} \hbox{~and~} var(\mathcal{P}_2(W^{n-1} \tilde{f}^{n-1}+b^{n-1})_i)=\frac{1}{m_{n-1}}.\]
  Since $\mathcal{P}_2(x)\approx ReLU(x)$, we have the variance of each component of the ($n-1$)-st layer to be $\frac{2}{m_{n-1}}$ \cite{he2015delving}.
Therefore we obtain
  	\begin{eqnarray}
 W^{n-1}_{j} \mathcal{P}_2(W^{n-2} \tilde{f}^{n-2}+b^{n-2}) + b^{n-1}_j \sim \mathcal{N}(0,{\frac{2}{m_{n-1}}})\nonumber\in \{W^{n-1}_{j}, b^{n-1}_j,1\} \times \{\mathcal{P}_2(W^{n-2} \tilde{f}^{n-2}+b^{n-2}), 1\}.\nonumber
	\end{eqnarray}
  Similarly, by solving the start system of the homotopy setup (\ref{linear_product1}), we have
  	\begin{equation}
A\left(
\begin{array}{c}
\big(W^{n-1}_{j}\big)^T\\
b^{n-1}_j
\end{array}\right)=-\alpha_3 \hbox{~and~} \beta_1\mathcal{P}_2(W^{n-2} \tilde{f}^{n-2}+b^{n-2})=-\beta_2,
	\end{equation} where  $\alpha_{3,i}\sim\mathcal{N}(0,{\frac{2}{m_{n-1}}})$. Then we have
\[var(W^{n-1}_{j,i})=var(b^{n-1}_j)=\frac{2}{m_{n-1}(m_{n-2}+1)} \hbox{~and~} var(\mathcal{P}_2(W^{n-2} \tilde{f}^{n-2}+b^{n-2})_i)=\frac{1}{m_{n-2}}.\] Therefore, in general, we conclude the variance for each component of weights and bias  on the $\ell$-th layer becomes
  	\begin{equation}
var(W^\ell_{j,i})=var(b_{j}^\ell)=\frac{2}{m_{\ell}(m_{\ell-1}+1)}.\end{equation}

Moreover, there are many solutions in the start system of the homotopy setup in (\ref{linear_product1}) by randomly choosing $\alpha_i$ and $\beta_i$. The number of solutions is determined by the degree of the polynomial system. More specifically,  on the complex plane, the degree of each variable on $\ell$-th layer based on the $\mathcal{P}_2$ polynomial approximation is $2^\ell$. Thus the total degree of the polynomial system is $1+2+2^2+\dots +2^n=2^{n+1}-1$.  Based on Bézout's  theorem \cite{bates2013numerically,SWbook}, the number of solutions is $2^{2^{n+1}-1}$ which grows exponentially as $n$ is large. Since the ratio of  solutions involving changes of weights and bias on the $\ell$-th layer to the whole solution set is  $\frac{2^\ell}{2^{n+1}-1}$,  we use this probability to re-initialize the weights on $\ell$-th layer until we have the desired loss. In theory, by re-initializing $2^{2^{n+1}-1}$ times, we can guarantee to compute all the local minima including the global minima; in practice, we may only need a few re-initialization before obtaining the desired loss. Finally, the LPS initiation process is summarized in {\bf Algorithm \ref{alg1}}.

\begin{algorithm}
 \KwData{The width of each layer $m_\ell$, $\ell=1,\dots, n$, $m_0=\dim(X)$.}
 \KwResult{Initialization of $\{W^\ell,b^\ell\},$ $\ell=1,\dots, n$ }
{\bf Step 1 (Initialization):}
% Initialize $\{W_1,b_1\}\sim\mathcal{N}(0,\frac{\sqrt{2}}{\sqrt{m_1(dim(X)+1)}})$;

 \For{$\ell=1,\dots,n-1$}{
 Initialize $\{W^\ell,b^\ell\}\sim\mathcal{N}(0,\frac{{2}}{{m_{\ell}(m_{\ell-1}+1)}})$;
}
 Initialize $\{W^n,b^n\}\sim\mathcal{N}(0,\frac{1}{m_{n-1}+1})$;

%\If{the initialization in {\bf Step 1} does not provide a }{
{\bf Step 2 (Re-initialization):}

Randomly choose an integer $d$ in $(0, 2^{n+1}-1)$;

 \For{$\ell=n,\dots,1$}{
$d_\ell=d\%2$ and $d=[d/2]$;

\If{$d_\ell==1$}
{initialize $\{W^\ell,b^\ell\}$ in {\bf Step 1} with a 50\% probability;}
  }

%}
 \caption{The LPS initialization algorithm}
 \label{alg1}
\end{algorithm}

\begin{remark}
 We use $d\%2$ and $[d/2]$ to denote the remainder and the quotient of $d$ by dividing $2$, respectively. In other words, we convert $d$ to a $n$-bit binary number. We now make three comments on the LPS initialization:
\begin{itemize}

\item {\bf Bias initialization: } The biases can also be initialized to zero because the gradients with respect to bias depend only on the linear activation of that layer not on the gradients of the deeper layers. Thus there is no diminishing or explosion of gradients for the bias terms.
  \item {\bf Stopping criteria:} If  the re-initialization does not improve the training loss, we may stop the initialization algorithm. Otherwise, we go to {\bf Step 2} for another re-initialization.
\item {\bf Apply to other activation functions:} Based on the  derivation of the LPS initialization, the variance in {\bf Algorithm \ref{alg1}} works for the ReLU activation function only. The variance of other activation functions needs to be further derived, for instance, the distribution for the tanh activation function is $\mathcal{N}(0,\frac{1}{{m_{\ell}(m_{\ell-1}+1)}})$.
\end{itemize}

\end{remark}

\section{Theoretical analysis of the LPS initialization on the dying ReLU}
The dying ReLU occurs when the weights are negative such that the ReLU neurons become inactive and remain to be zero for any input \cite{lu2019dying}. Therefore, the gradient is zero so that large parts of the neural network do nothing. If the neural networks are deep, the dying ReLU may even occur at the initialization step and the whole training process based on existing initialization algorithms  fails at the very beginning. The LPS initialization strategy resolves this issue with a theoretical guarantee by the following two theorems.

\begin{theorem} \label{thm:upper-bound-re-ini}
{\bf (One re-initialization)}
	If a ReLU feed-forward neural network $y(x,\theta)$ with $n$ layers,
	each having width $m_1, \dots, m_n$,
	is re-initialized {\underline{once}} by the LPS initialization, then the probability of the dying ReLU occurring is
	\begin{equation}\label{pro}
	P\left(  y(x; {\theta}) \text{ is born dead in } \Omega \right) \le
	1 - \prod_{\ell=1}^{n-1} \left(1 -
	{2^{-m_\ell}\left((1-p_\ell) + p_\ell (1 - \frac{\delta}{2} )^{m_\ell} \right)}\right),
	\end{equation}
	where $\delta\leq \frac{1}{2}$ is a constant independent of $\ell$ and $p_\ell=\frac{2^\ell}{2^{n+1}-1}$ is the probability of choosing the $\ell$-th layer to re-initialize. Here $\Omega=[-c,c]^{m_0}$, $\forall c>0$.
\end{theorem}
\begin{proof}
See the detailed proof in Appendix \ref{pr4.1}.

\end{proof}

\begin{theorem}\label{thm:upper-bound-re-ini-N}
{\bf (Multiple re-initialization)}
	If a ReLU feed-forward neural network $y(x,\theta)$ with $n$ layers	is re-initialized \underline{$N$ times} by the LPS initialization, then the probability of the dying ReLU occurring is
	\begin{equation}
P\left(  y(x; {\theta})  \text{ is born dead in } \Omega \right) \le 1 - \prod_{\ell=1}^{n-1} \left( 1 -  \frac{M_\ell}{2} \left( 1 -  \frac{p_\ell}{4}\right)^N \right)~ \rightarrow 0 \hbox{ (as $N\rightarrow\infty$),}\label{ntime}
	\end{equation}
where $M_\ell = m_\ell \times (m_{\ell-1} + 1)$.
\end{theorem}
\begin{proof}
See the detailed proof in Appendix \ref{pr4.2}.

\end{proof}
\begin{remark}
 If there is no re-initialization, then we have $p_\ell=0$. In this case, we can further simplify (\ref{pro}) as $\displaystyle P\left(  y(x; {\theta}) \text{ is born dead in } \Omega \right) \le
    1 - \prod_{i=1}^{n-1} \left(1 -2^{-m_i}\right)$, by assuming $m_i=m$, we have 
  \[
\lim_{n\rightarrow \infty}P\left(  y(x; {\theta}) \text{ is born dead in } \Omega \right) =1 \hbox{~and~} \lim_{m\rightarrow \infty}P\left(  y(x; {\theta}) \text{ is born dead in } \Omega \right) =0,\]
which means the dying ReLU must occur for deep neural networks ($n\rightarrow \infty$) without  re-initialization such as existing random initialization algorithms \cite{lu2019dying}. In order to tolerate the large number of layers, we have to increase the width of each layer ($m\rightarrow\infty$) to make the probability smaller. Moreover, based on {\bf Theorem} \ref{thm:upper-bound-re-ini}, we can significantly decrease the rate of dying ReLUs in a network with only one re-initialization.

For multiple re-initialization, the LPS initialization guarantees theoretically that the ReLU networks never die with probability one when the number of re-initialization goes to infinity.  However, on the other hand, when $N\rightarrow \infty$, almost all the initialized weights become positive which reduces the ReLU neural network to a linear neural network. Then the training will not benefit from this weight initialization. Therefore, due to the exponential decay in  (\ref{ntime}), only a few time re-initialization will make the initialized weights optimal for training.

\end{remark}

\section{Numerical Experiments}
In this section,  we apply the LPS initialization algorithm to both fully connected neural networks and convolutional neural networks with the ReLU activation function and compare it with the He initialization developed in \cite{he2015delving}.    {{All the experimental details and hyperparameters are reported in {\bf Appendix 6.5}.}}

\subsection{Fully Connected Neural Networks}
\begin{figure}
        \centering
         \includegraphics[width=0.42\textwidth]{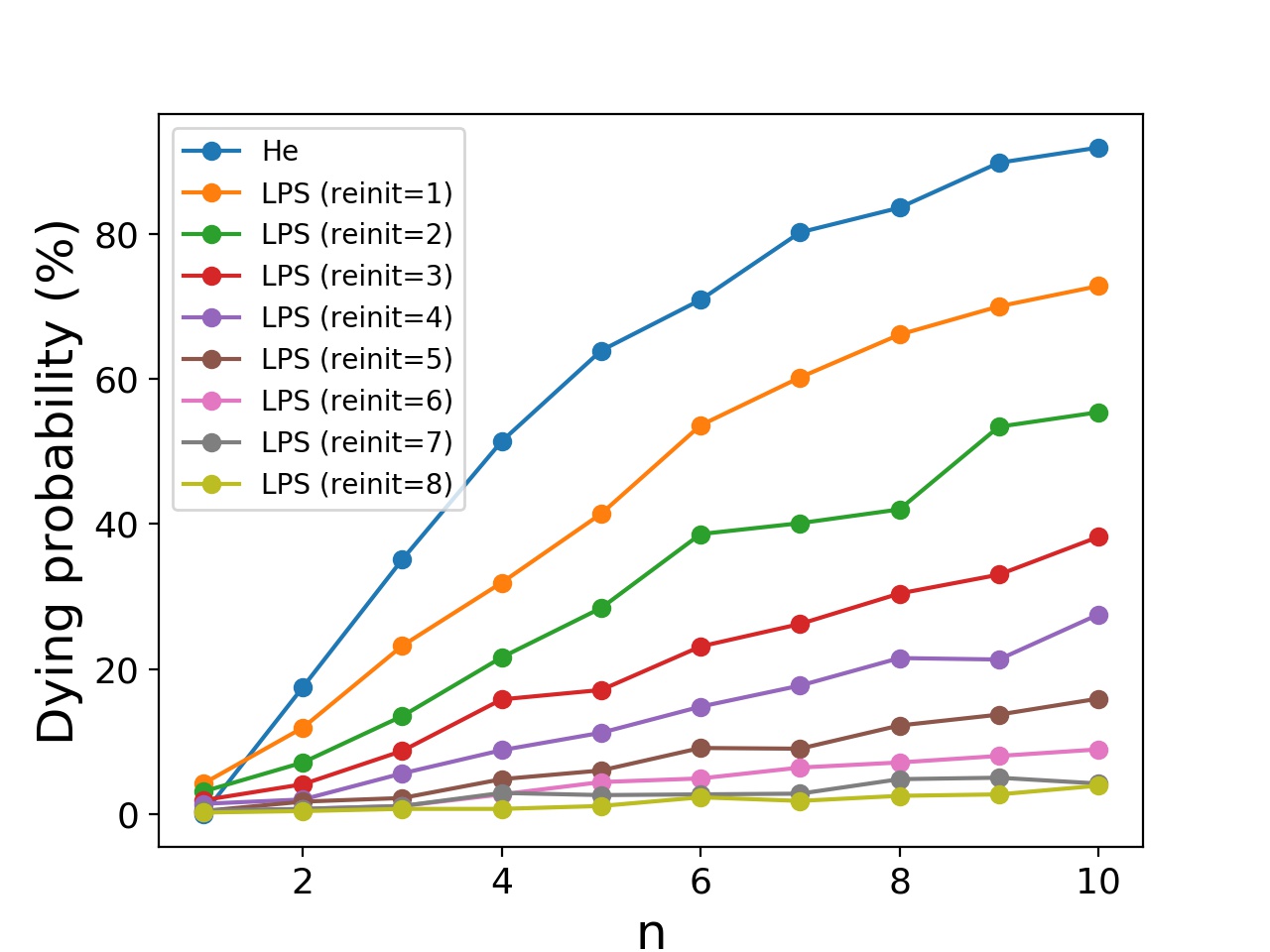}
          \includegraphics[width=0.42\textwidth]{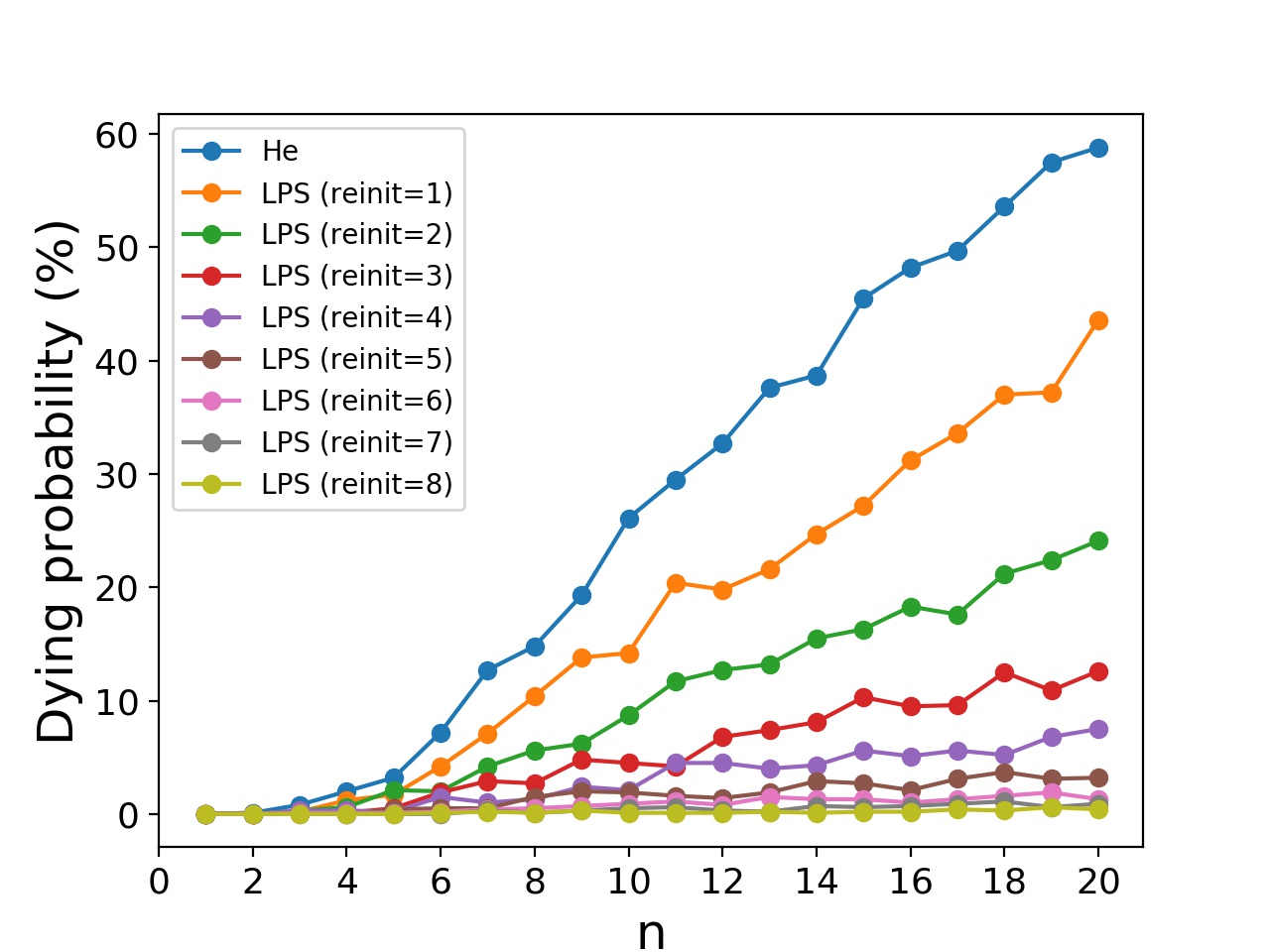}
          \caption{Probability of born dying ReLU on fully connected neural networks v.s. the number of layers, $n$, for different initialization algorithms based on 1000 initializations (Left: the 1D network of width 2; Right: the 2D network of width 4). We evaluate the neural networks on the uniform grid points in $[-1,1]^{m_0}$ with a stepsize $0.1$, calculate the variance of $y(x,\theta)$, and count as the ``dying ReLU'' if the variance is less than $10^{-10}$.}
          \label{fc_dying_prob}
\end{figure}
Based on our analysis, the dying ReLU occurs frequently in deep neural networks. In the first example, we compare our initialization algorithm with
He initialization on deep fully connected neural networks by approximating different functions. More specifically, we employ a 10-hidden-layer ReLU neural network with
hidden width of 2 to approximate the following 1D functions:
\begin{equation}
	f_1(x) = |x|, \quad f_2(x) = x\sin(5x), \quad f_3(x) = 1_{\{x>0\}} +
	0.2\sin(5x),
\end{equation}
and use a 20-hidden-layer ReLU neural network with hidden
width of 4 (both the input and output size are 2) to approximate the following 2D function:
\begin{equation}
	f_4(x_1,x_2) = \begin{bmatrix} |x_1+x_2|\\|x_1-x_2|\\ \end{bmatrix}.
\end{equation}
First, we use different initialization algorithms 1000 times and show the probability of born dying ReLU in Fig \ref{fc_dying_prob} by evaluating the networks only. It clearly shows that the LPS initialization has a much lower probability of a born dead neural network compared to He initialization.
Second, we train different neural networks by these two initialization algorithms to see their effects on the training process.   When the dying ReLU occurs, the training result becomes a flat line (x-axis) in 1D and a flat plane ($x-y$ plane) in 2D. We call this the ``collapse'' case. But even for the non-collapse case, we cannot guarantee the loss goes to zero since the training algorithm may be trapped in a local minimum (See Fig. \ref{f1_approx} for four non-collapse results which are randomly chosen from all non-collapse results of each function). Our algorithm focuses on the initialization part, not the training process. Therefore, we initialize $1000$ times for each initialization algorithm (up to 8 re-initialization) and count how many times the training does not collapse which is shown in Table \ref{table:fc}. It is obvious that the LPS initialization has a much lower chance with the dying ReLU during the training. Moreover, with only 6-7 re-initialization for $f_1$ and $f_2$, the LPS initialization achieves the optimal performance which confirms the conclusion of {\bf Theorem} \ref{thm:upper-bound-re-ini-N}.

\begin{figure}
        \centering
        \includegraphics[width=0.41\textwidth]{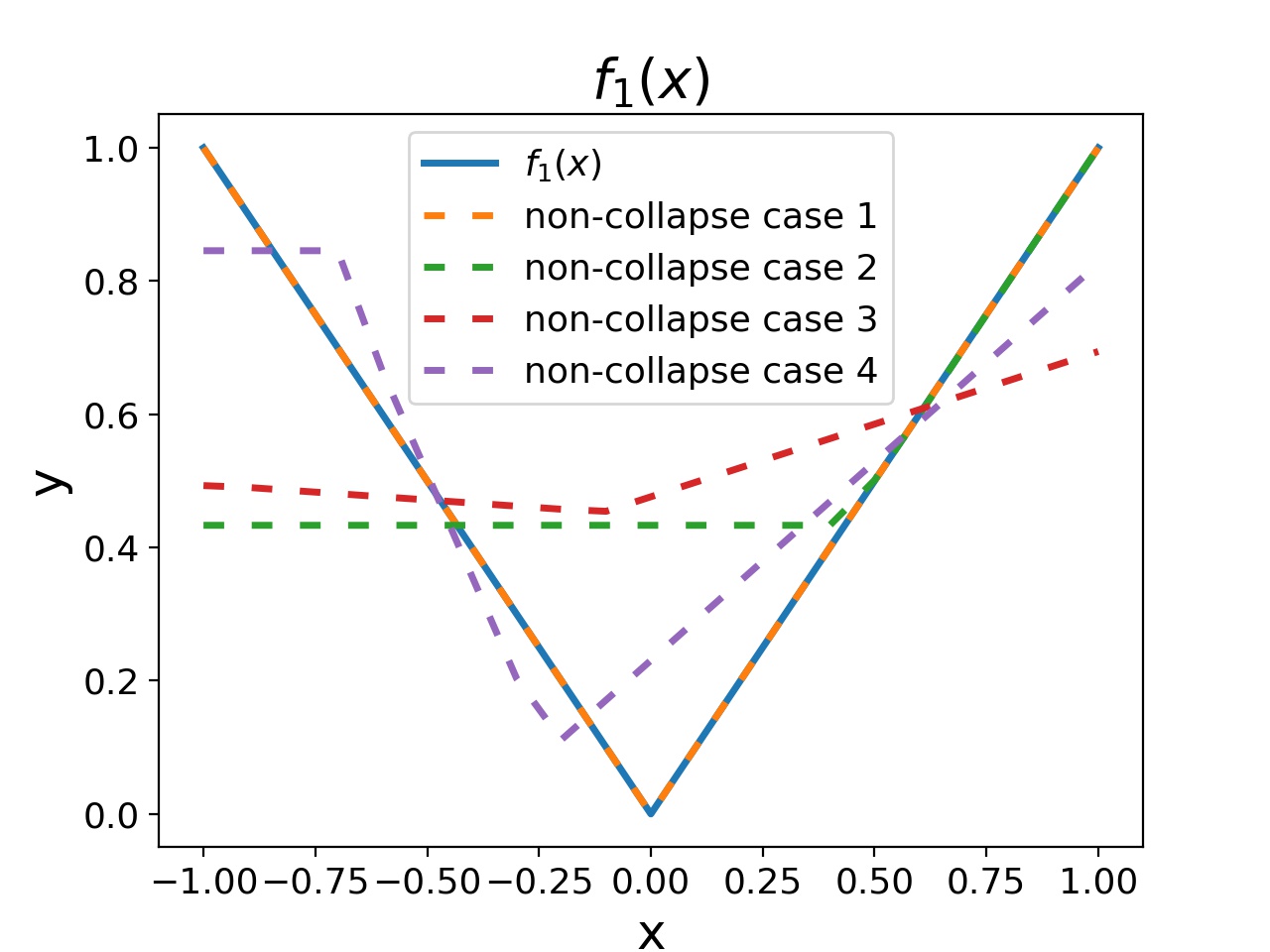}
              \includegraphics[width=0.41\textwidth]{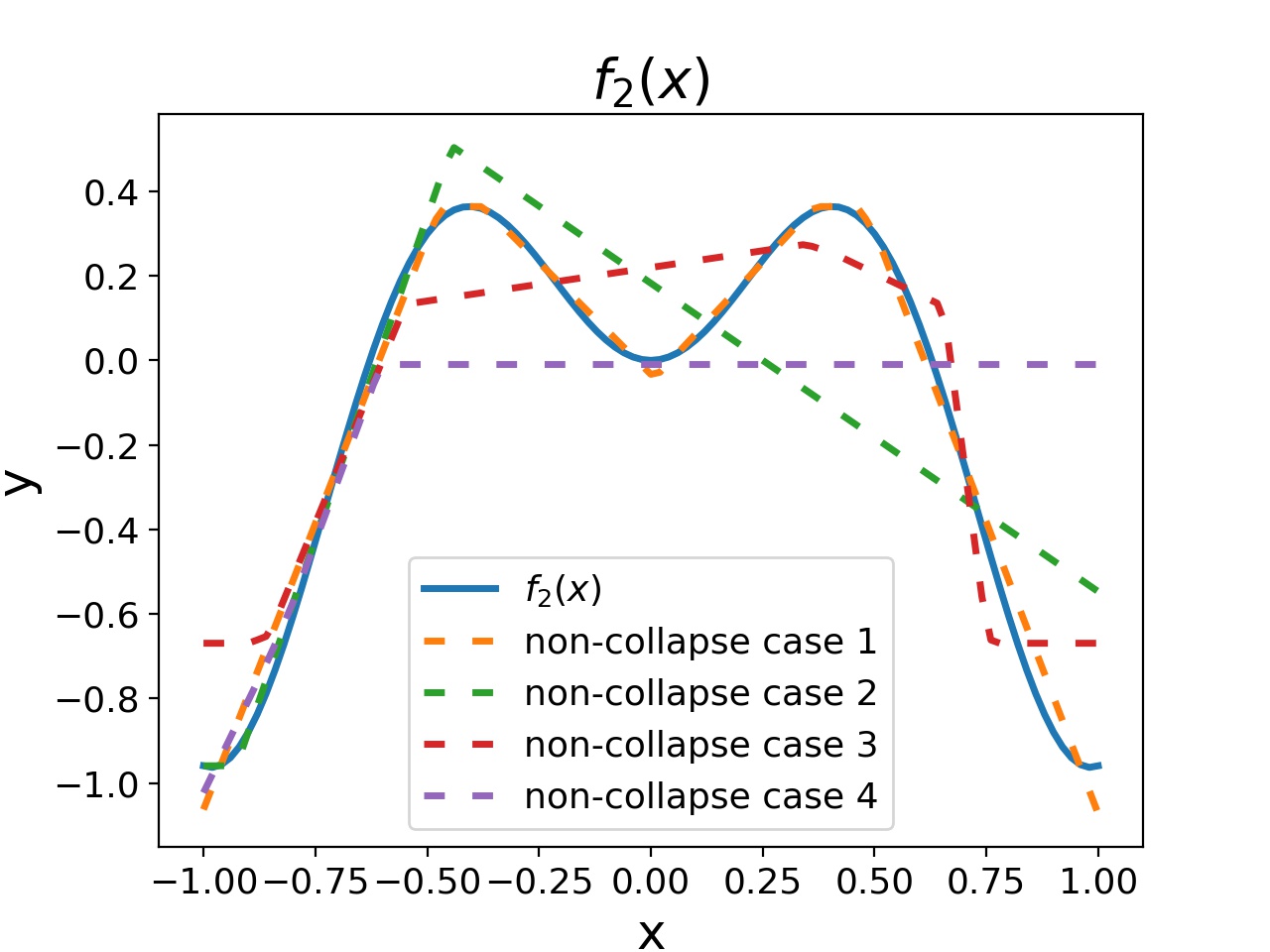}
        \includegraphics[width=0.41\textwidth]{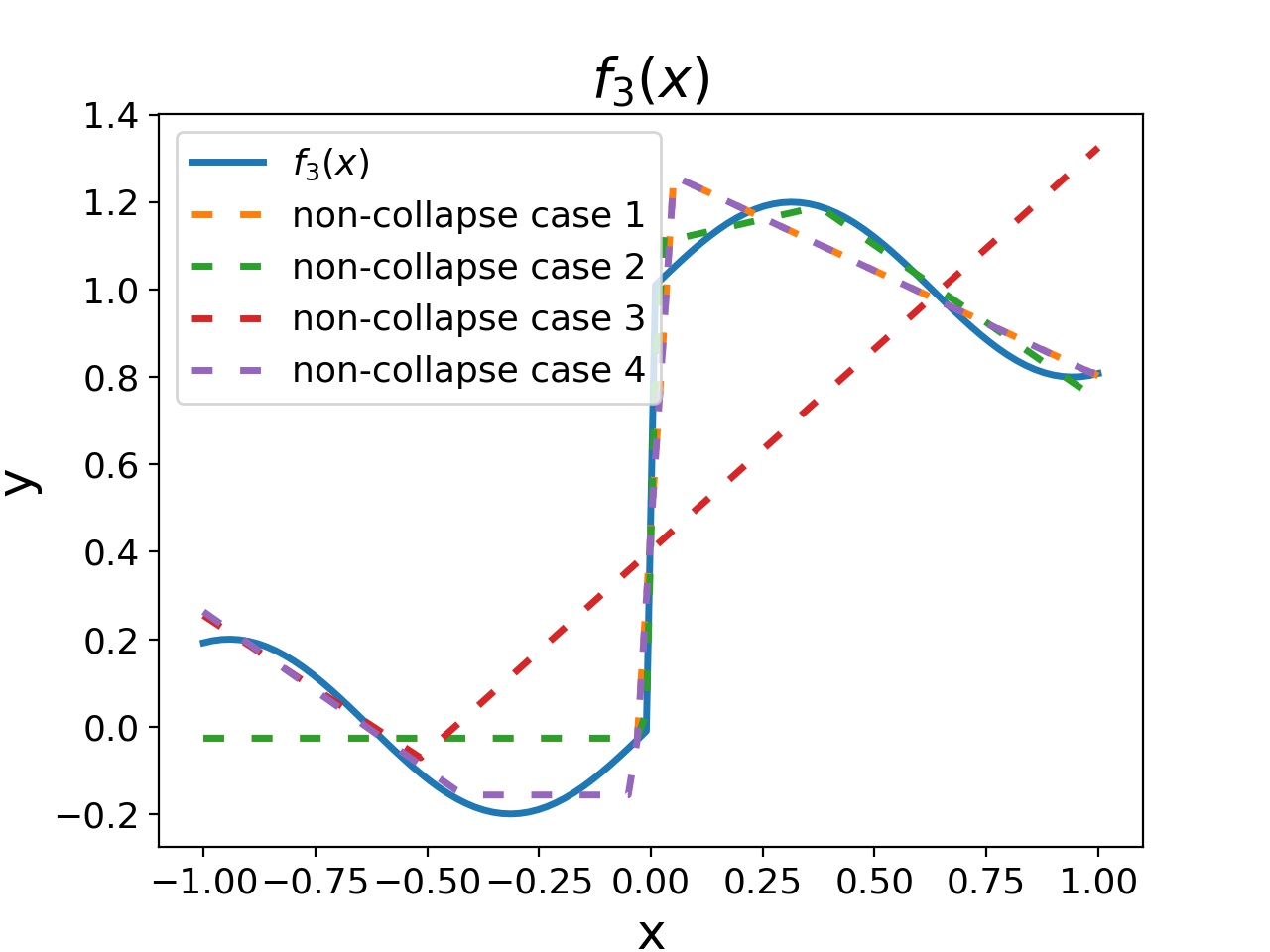}
        \includegraphics[width=0.41\textwidth]{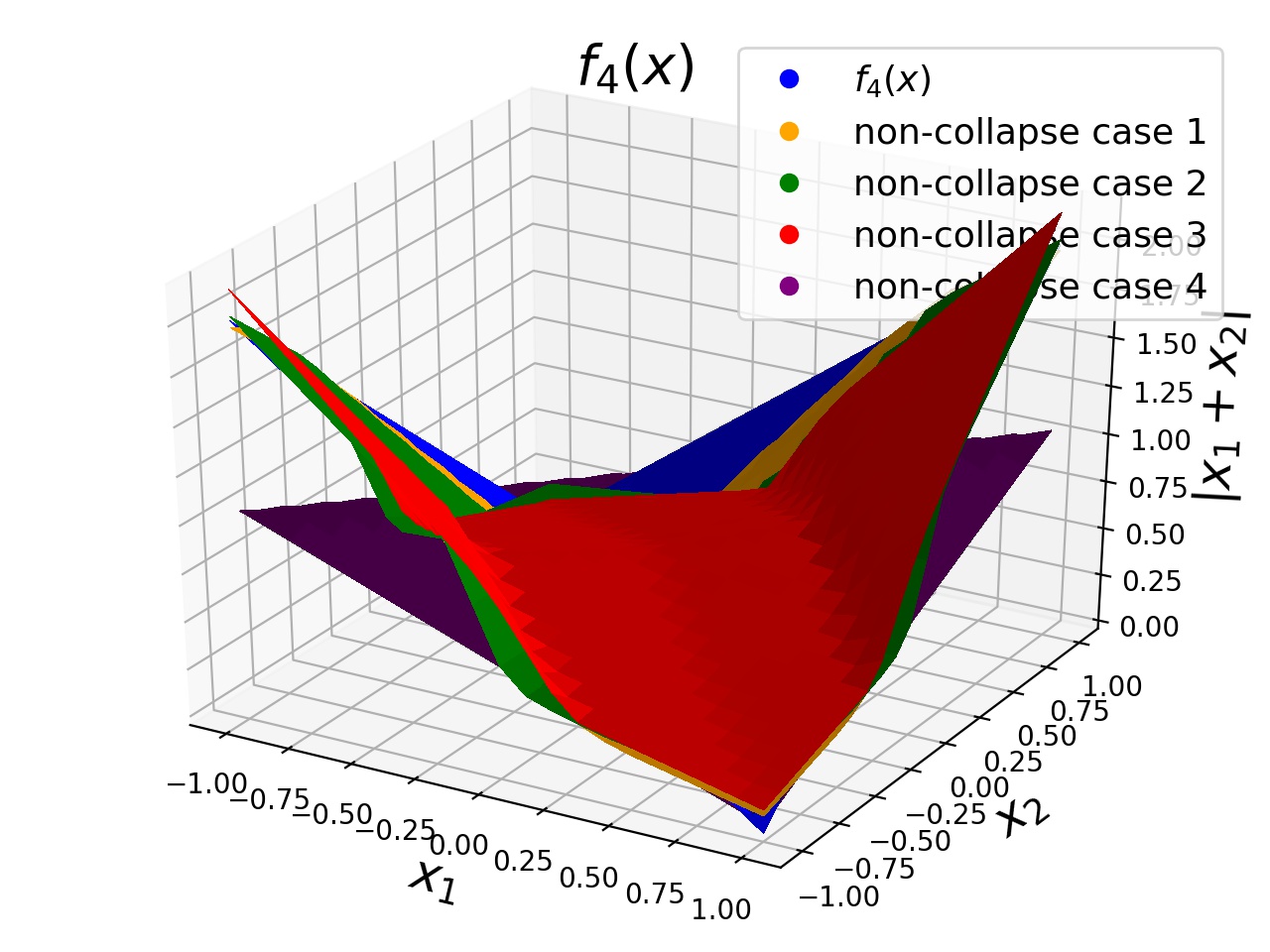}
          \caption{``Non-collapse'' approximation results by neural networks for $f_1-f_4$.}
        \label{f1_approx}
\end{figure}

\begin{table}[h]
\scriptsize
		\centering
	\caption{The percentage of finding ``non-collapse'' cases among 1,000 initialization for different initialization algorithms. ``Non-collapse'' cases for different functions are shown in Fig. \ref{f1_approx}. (The optimal LPS initialization is highlighted)}
		\begin{tabular}{|c|c|c|c|c|c|c|c|c|c|}%c|c|c|c|}
			\toprule
			function & He initialization &  \multicolumn{8}{|c|}{LPS initialization}\\
			\cmidrule{3-10}
			 &  &reinit=1 & reinit=2 & reinit=3 & reinit=4 & reinit=5& reinit=6 & reinit=7 & reinit=8\\% & reinit=9 & reinit=10 & reinit=11 & reinit=12\\
			\midrule
			  $f_1$ & 4.5\% & 9.5\% &18.8\% &28.1\% & 37.4\% & 40.2\% & 37.0\% & \bf{40.4\%} & 38.7\% \\%& 36.6\% & 34.8\% & 35.8\% & 32.8\%\\
			 $f_2$ & 5.6\% & 8.7\% &15.8\% & 22.1\% & 22.3\% & 21.8\% & \bf{22.7\%} & 22.3\% & 20.8\%\\%& 21.9\% & 18.9\% & 18.7\% & 19.2\%\\
			 $f_3$ & 3.2\% &12.4\% &29.2\% & 43.6\% & 58.0\% & 74.1\% & 81.9\% & 88.2\% & \bf{92.1\%} \\%& 95.1\%   & 96.7\% & \bf{98.1\%} & 98.0\%\\
			 $f_4$ &22.9\% & 38.7\% & 60.5\% & 75.1\% & 85.3\% & 92.7\% & 96.5\% & 98.3\% & \bf{98.9\%} \\%& 99.2\% & 99.4\% & 99.6\% & \bf{99.9\%}\\
			\bottomrule
		\end{tabular}
	\label{table:fc}
\end{table}

\subsection{Convolutional neural networks}
We apply the LPS initialization to various benchmark models on public datasets. We use the cross-entropy as the loss function and
the stochastic gradient descent (SGD) as the training algorithm.
%\vspace{-0.2in}

\subsubsection{LeNet networks on the MNIST dataset}
First, we benchmark  three different LeNet networks, LeNet-1, LeNet-4, and LeNet-5, on  the MNIST dataset  \cite{MNIST} to compare the
LPS initialization with He initialization. Each initialization algorithm is used to generate 100 different initialization (up to three re-initialization for the LPS initialization).  The error rates on the validation dataset for all the initialization algorithms are summarized in  Table
\ref{table:lenet}. The LPS initialization has lower error rates on both LeNet-4 and LeNet-5 and achieves a comparable error rate on LeNet-1 with He initialization.  On LeNet-1, the LPS initialization has a larger deviation since it can compute all the local minima but some of them may have high generalization errors. In order to quantify the performances of different algorithms further, we define  Good Local Minimum Percentage (GLMP) as the percentage of the validation accuracy greater
than 99\% in 100 initializations and show in Table
\ref{table:lenet}. We find that the LPS initialization (with the optimal number of re-initialization) has a much better chance to find a good local minimum on all three models. Moreover, the LPS initialization achieves the optimal performance with only one re-initialization for LeNet-1 and LeNet-4 and with two re-initialization for LeNet-5.
\begin{table}[h]
		\scriptsize
		\centering
	\caption{Comparison between He initialization and the LPS  initialization on the MNIST dataset by training three LeNet networks. All the results are based on 100 initializations for each algorithm. The error rate is mean$\pm$std while GLMP is the percentage of the validation accuracy greater
than 99\%. (The optimal LPS initialization is highlighted)}
		\begin{tabular}{|c|c|c|c|c|c|c|c|c|}
			\toprule
			 Network &  \multicolumn{2}{|c|}{He initialization} &  \multicolumn{6}{|c|}{LPS initialization}\\
			\cmidrule{4-9}
			 &  \multicolumn{2}{|c|}{} & \multicolumn{2}{|c|}{reinit=1} & \multicolumn{2}{|c|}{reinit=2} & \multicolumn{2}{|c|}{reinit=3}\\
			\cmidrule{2-9}
			  & GLMP & error rate & GLMP & error rate & GLMP & error rate & GLMP & error rate\\
			  \cmidrule{1-9}
			  LeNet-1 & 24\% & 1.14$\pm$0.21\% & \bf{29\%} & \bf{1.99$\pm$8.71\%} & 20\% & 2.10$\pm$8.71\% & 16\% & 2.08$\pm$8.70\%\\
			  \midrule
			  %LeNet-4 & GLMP & error rate & GLMP & error rate & GLMP & error rate & GLMP & error rate\\
			  %\cmidrule{2-9}
			  LeNet-4& 89\% & 4.30\%$\pm$17.21\% & \bf{95\%} & \bf{2.45$\pm$12.31\%} & 89\% & 6.86$\pm$22.43\% & 87\% & 8.61$\pm$25.16\%\\
			  \midrule
			  %LeNet-5 & GLMP & error rate & GLMP & error rate & GLMP & error rate & GLMP & error rate\\
			 % \cmidrule{2-9}
			  LeNet-5& 84\% & 13.94$\pm$31.38\% & 89\% & 10.29$\pm$27.54\% & \bf{93\%} & \bf{6.75$\pm$22.46\%} & 92\% & 6.76$\pm$22.46\%\\
			\bottomrule
		\end{tabular}
	\label{table:lenet}
\end{table}

\subsubsection{VGG networks on the CIFAR-10 dataset}
Secondly, we compare the LPS initialization with He initialization for different VGG networks, VGG9, VGG11, VGG13, and VGG16, on the CIFAR-10 dataset \cite{CIFAR10}.  We run 10 initialization for each of them and show the comparisons in Table \ref{table:vgg}. All the error rates of the LPS initialization are comparable with He initialization but the lowest error rate (LER) of the LPS initialization is lower than He initialization. This confirms that the LPS initialization can find local minima in a larger solution landscape (with a larger deviation and lower LER) while He initialization focuses on a smaller solution landscape. The LPS initialization achieves the optimal performance of up to three re-initialization for all four VGG networks.

\begin{table}[h]
		\scriptsize
		\centering
	\caption{Comparison between He initialization and the LPS  initialization on the CIFAR-10 dataset by training four VGG models. All the results are based on 10 initialization for each algorithm. The error rate is mean$\pm$std while LER is the lowest error rate among the 10 runs.}
		\begin{tabular}{|c|c|c|c|c|c|c|c|c|}
			\toprule
			 Network &  \multicolumn{2}{|c|}{He initialization} &  \multicolumn{6}{|c|}{LPS initialization}\\
			\cmidrule{4-9}
			 &  \multicolumn{2}{|c|}{} & \multicolumn{2}{|c|}{reinit=1} & \multicolumn{2}{|c|}{reinit=2} & \multicolumn{2}{|c|}{reinit=3}\\
			\cmidrule{2-9}
			  & LER & error rate & LER & error rate & LER & error rate & LER & error rate\\
			  \cmidrule{1-9}
			  VGG9 & 5.95\% & 6.09$\pm$0.10\% & 5.89\% & \bf{6.02$\pm$0.06\%} & \bf{5.82\%} & 6.06$\pm$0.13\% & 5.84\% & 6.11$\pm$0.16\%\\
			  \midrule
			  VGG11& 7.34\% & 7.62$\pm$0.17\% & 7.51\% & 7.89\%$\pm$0.21\% & 7.42\% & \bf{7.82$\pm$0.20\%} & \bf{7.17\%} & 7.95$\pm$0.36\%\\
			  \midrule
			  VGG13& 5.84\% & 5.96$\pm$0.06\% & \bf{5.73\%} & \bf{6.01\%$\pm$0.16\%} & 5.84\% & 22.87$\pm$33.56\% & 6.12\% & 39.88$\pm$40.92\%\\
			  \midrule
			  VGG16& 6.25\% & 6.39\%$\pm$0.10\% & 6.15\% & 14.74$\pm$25.08\% & \bf{5.98\%} & \bf{6.49$\pm$0.26\%} & 6.23\% & 14.84$\pm$25.05\%\\
			\bottomrule
		\end{tabular}
	\label{table:vgg}
\end{table}

\subsubsection{ResNets on different datasets}
Last, we validate the LPS initialization on different ResNets with public datasets.  The ResNet architecture \cite{he2016deep} does not have the vanishing gradient problem since skip connections act as gradient superhighways, allowing the gradient to flow unhindered. That's the main reason why He initialization performs very well on the ResNets.  We compare the LPS initialization with He initialization on different ResNets on CIFAR-10 and CIFAR-100 \cite{CIFAR10}. Results in Table \ref{table:resnet} show that the LPS initialization achieves better error rates than He initialization. Fig. \ref{error:resnet} shows that the LPS initialization has a comparable training performance to He initialization on the ImageNet dataset \cite{deng2009imagenet}.

%: the lowest error rates of the LPS initialization are 30.54\% and 32.05\% (reinit=2) while that of He initialization
%is 30.78\%  and 32.82\% on ResNet34 and ResNet56, respectively.

\begin{table}		\tiny
		\centering
	\caption{Comparison between He initialization and the LPS  initialization on the CIFAR-10 and CIFAR-100 datasets by training different ResNet models with 10 initialization for each algorithm. %Results of ImageNet are based on 1 initialization.
The error rate is mean$\pm$std while LER is the lowest error rate among the 10 runs.}
		\begin{tabular}{|c|c|c|c|c|c|c|c|c|c|}
			\toprule
			Dataset& Network &  \multicolumn{2}{|c|}{He initialization} &  \multicolumn{6}{|c|}{LPS initialization}\\
			\cmidrule{5-10}
			 &&  \multicolumn{2}{|c|}{} & \multicolumn{2}{|c|}{reinit=1} & \multicolumn{2}{|c|}{reinit=2} & \multicolumn{2}{|c|}{reinit=3}\\
			\cmidrule{3-10}
			  && LER & error rate & LER & error rate & LER & error rate & LER & error rate\\
			  \cmidrule{1-10}
			  			\multirow{3}{*}{CIFAR-10}&ResNet-20 & 7.28\% & 7.50$\pm$0.15\% & 7.39\% & \bf{7.78$\pm$0.24\%} & \bf{7.20\%} & 7.88$\pm$0.23\% & 7.28\% & 7.79$\pm$0.26\%\\
			  \cmidrule{2-10}
			  &ResNet-56 & 6.26\% & 6.72$\pm$0.32\% & 6.06\% & \bf{6.49$\pm$0.34\%} & \bf{6.00\%} & 6.74$\pm$0.41\% & 6.29\% & 31.66$\pm$38.18\%\\
			  \cmidrule{2-10}
			  &ResNet-110& 5.93\% & 6.34$\pm$0.37\% & 5.76\% & 6.49\%$\pm$0.48\% & \bf{5.70\%} & 6.38$\pm$0.37\% & 6.06\% & \bf{6.37$\pm$0.23\%}\\
			  \cmidrule{1-10}
			  CIFAR-100&ResNet-164 & 26.36\% & 28.99$\pm$1.92\% & \bf{26.18\%} & 28.30$\pm$2.18\% & 26.79\% & \bf{27.68$\pm$0.53\%} & 27.15\% & 28.25$\pm$0.94\%\\
%			  \cmidrule{1-10}
%			  \multirow{2}{*}{ImageNet}&ResNet-34 & 30.68\% & - & \bf{30.54\%} & - & 30.78\% & - & 99.9\% & -\\
%			   \cmidrule{2-10}
%			  &ResNet-50 & 32.82\% & - & 32.66\% & - & \bf{32.05\%} & - & 32.63\% & -\\
			  \bottomrule
		\end{tabular}
	\label{table:resnet}
\end{table}

%\begin{figure}\vspace{-0.2in}
\begin{figure}
        \centering
        \includegraphics[width=0.45\textwidth]{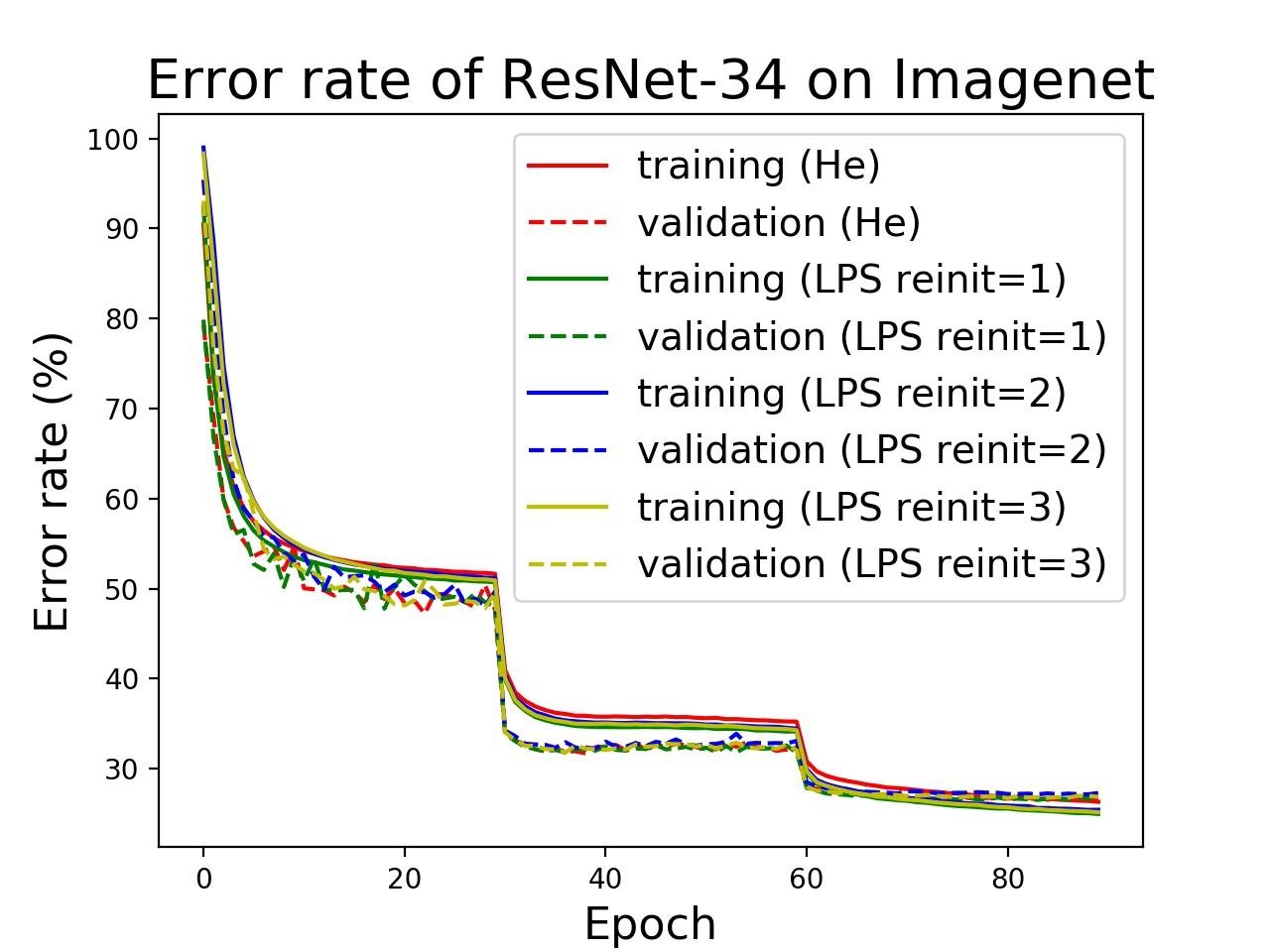}
        \includegraphics[width=0.45\textwidth]{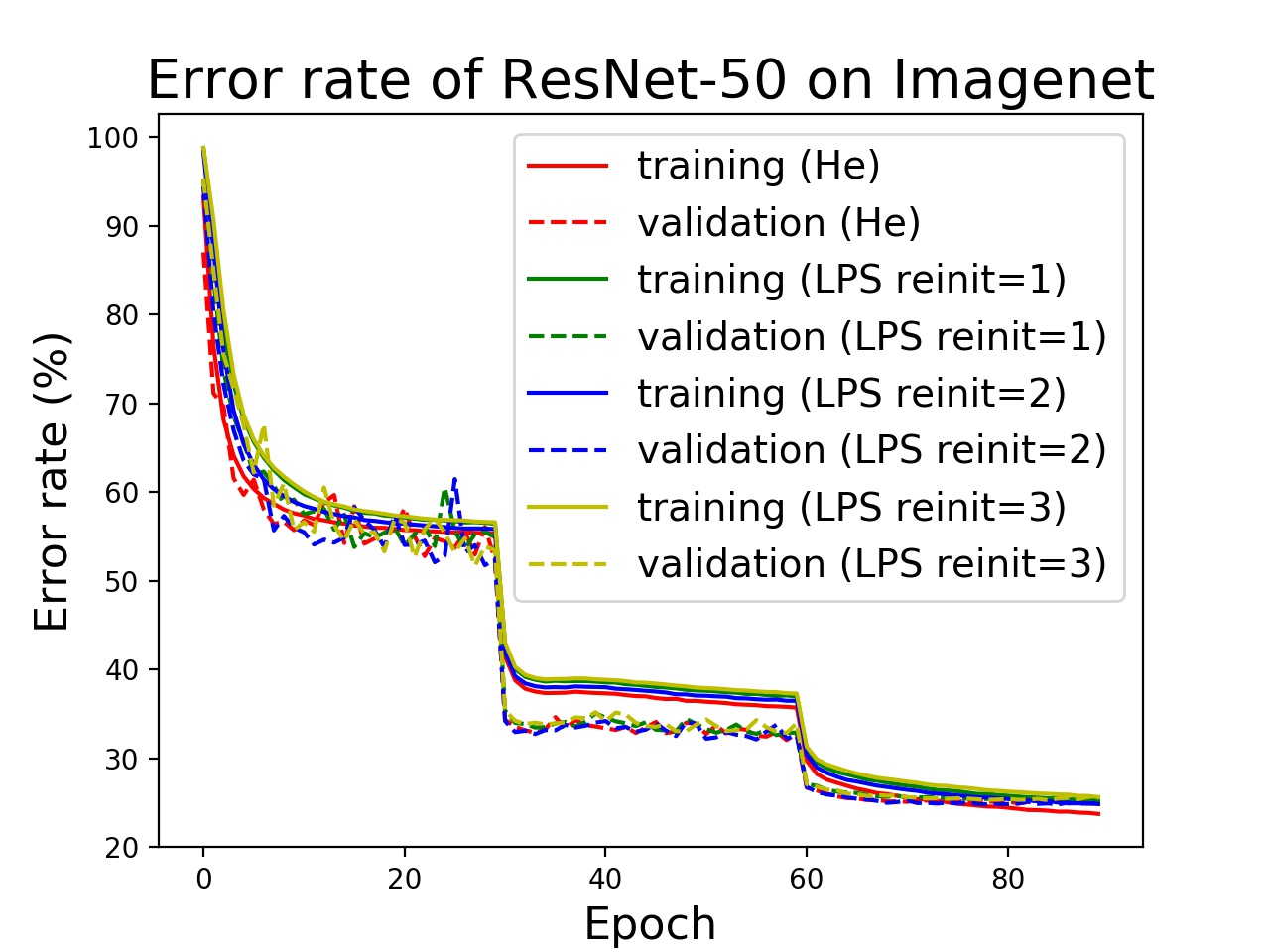}
          \caption{Error rates of the LPS initialization and He initialization v.s. Epoch on the ImageNet dataset. {\bf Left:} Error rates of the LPS initialization are  {\bf 26.54\%} (reinit=1),  27.15\% (reinit=2) and 26.73\% (reinit=3) comparing to 26.55\% with He initialization on the ResNet-34; {\bf Right:} Error rates of the LPS initialization are 25.15\% (reinit=1),  {\bf 24.80\%} (reinit=2) and 25.28\% (reinit=3) comparing to 24.91\% with He initialization on the ResNet-50.}
        \label{error:resnet}
\end{figure}

%\[f^{\ell}=\sum_{i=0}^n\alpha_i\sigma^i(W^\ell f^{\ell-1}+b^\ell), \]
%If $n=1$, the we recover the Residual network, namely, \[f^{\ell}=\alpha_1\sigma(W^\ell f^{\ell-1}+b^\ell)+\alpha_0f^{\ell-1}. \]

\section{Conclusion}
Weight initialization is crucial for efficiently training neural networks and therefore has become an active research area in machine learning. Current existing initialization algorithms are based on minimizing the variance of parameters between layers and lack  consideration of the nonlinearity of neural networks which is the most essential part of the training. In this paper, we analyze the nonlinearity of neural networks from a nonlinear computation point of view and develop a novel initialization procedure based on the linear product structure of the neural network via approximating the ReLU activation function by polynomials. The LPS initialization can guarantee to find all the local minima based on the theory of numerical algebraic geometry and avoid the dying ReLU issue with the probability one from a theoretical perspective. Then we test this new initialization procedure on various benchmark networks and public datasets to show its feasibility and efficiency. We hope the LPS initialization complements current random initialization algorithms especially when the dying ReLU occurs. The theoretical analysis of the LPS initialization is for feedforward networks only and will be extended to the recurrent setting in the future. We will also further explore the LPS initialization to other activation functions and parameter normalization to improve efficiency.

%\bibliographystyle{RS}
%\bibliography{ref}
\section*{Appendix}
\subsection{Proof of Theorem 2.1}
\begin{proof}
The Sobolev space $H^k(\Omega)$ is defined as
\[H^k(\Omega)=\Big\{\int_\Omega |u(x)|^2dx<\infty, \int_\Omega |u^{(i)}(x)|^2dx<\infty,\forall i\leq k\Big\}\]
with
\[\|u\|_{H^k(\Omega)}=\sqrt{\sum_{i=0}^k\int_\Omega |u^{(i)}(x)|^2dx}.\]
Then the  estimate of the Legendre polynomial approximation~\cite{xiu2002wiener} is : if $\sigma \in H^k(-1,1)$, $k > 0$, we have
\[	\|\sigma - P_d \sigma \|_{L^2} \le C d^{-k} \|\sigma\|_{H^k}.\]
Since $ {\rm ReLU}(x)\in H^1(-1,1)$, we have the following approximation rate
\begin{equation}
\|\sigma - P_d \sigma \|_{L^2(-1,1)} \le C\frac{1}{d} \|\sigma\|_{H^1(-1,1)}.
\end{equation}

Next we consider the optimal Legendre polynomial approximation for $\sigma(x) = {\rm ReLU}(x)$ on $[-1,1]$ in the $L^{2}$ norm.
The Legendre's differential equation is given by
\begin{equation}
\frac{d}{dx}\left[ (1-x^2) \frac{d }{dx}L_k(x)  \right] = -k(k+1)L_k(x), \quad x\in (-1,1),
\end{equation}
for any $k \ge 2$.
By multiplying $x$ and integrating from $0$ to $1$ on both sides, we have
\begin{equation}\label{eq:xLk}
\int_{0}^1 x \frac{d}{dx}\left[ (1-x^2) \frac{d }{dx}L_k(x)  \right]  dx = \int_{0}^1 -k(k+1) xL_k(x) dx.
\end{equation}
Since
\begin{align*}
&\int_{0}^1 x \frac{d}{dx}\left[ (1-x^2) \frac{d }{dx}L_k(x)  \right]  dx\\
&=  -\int_{0}^1 (1-x^2)\frac{d}{dx}L_k(x) dx + \left.x(1-x^2)\frac{d}{dx}L_k(x)\right|_0^1 \\
&=\int_{0}^1 \frac{d(1-x^2)}{dx} L_k(x) dx - \left.(1-x^2)L_k(x)\right|_0^1 \\
&= -2 \int^1_0 x L_k(x)dx + L_k(0),
\end{align*}
we simplify \eqref{eq:xLk} as
\begin{equation}
\int_{0}^1 x L_k(x) dx = \frac{L_k(0)}{2 - k(k+1)},
\end{equation}
where
\begin{equation}
L_k(0) = \begin{cases}
\frac{(-1)^m}{4^m} \binom{2m}{m}, \quad &\text{for} \quad k = 2m,\\
0, \quad &\text{for}\quad k = 2m+1,
\end{cases}
\end{equation}
Moreover,
\begin{equation}
\|L_k(x)\|^2_{L^2(-1,1)} = \frac{2}{2k+1},
\end{equation}
then the coefficient of the optimal Legendre polynomial becomes
\begin{equation}
 \alpha_k = \frac{\int_{-1}^1 \sigma(x) L_k(x) dx}{\|L_k\|^2_{L^2(-1,1)}}= \frac{\int_0^1 x L_k(x) dx}{\|L_k\|^2_{L^2(-1,1)}}=\begin{cases}
\frac{(-1)^m(2k+1)}{2(2-k(k+1))4^m} \binom{2m}{m} \quad &\text{for} \quad k = 2m,\\
0, \quad &\text{for}\quad k = 2m+1,
\end{cases}
\end{equation}
for any $m \ge 1$.

\end{proof}

\subsection{Solving the polynomial system  by homotopy continuation method}\label{sec:NE}
The quadratic and fourth degree approximated polynomials have the following explicit formulas
\begin{equation}\label{keyP2}
P_2 \sigma = \frac{15}{32}x^2 + \frac{1}{2}x + \frac{3}{32} \hbox{~and~}P_4 \sigma =  -\frac{105}{256} x^4 +\frac{105}{128}x^2 +\frac{1}{2}x + \frac{15}{256}.
\end{equation}
The comparison between the ReLU activation function and $P_2\sigma(x)$ and $P_4\sigma(x)$ is shown in Fig. \ref{Fig:p2}.

We consider  polynomial systems arising from  fully connected neural networks to fit a 1D function $y=|x|$.

\noindent{\bf One-hidden-layer neural network:} We employ a one-hidden-layer ReLU network with width 2, namely,	\begin{equation}\label{1hidden}
y(x,\theta)=w^2 \text{ReLU}(w^1x+b^1)+b^2,
	\end{equation}
where $w^2\in R^{1\times2}, w^1\in R^{2\times1}, b^1\in R^{2\times1}$, and $b^2\in R$, and $\theta=(w^1_1,w^1_2,w^2_1,w^2_2,b^1_1,b^1_2,b^2)^T$. By using the $P_2 ReLU(x)$ activation function, then $y(x,\theta)$ becomes a quadratic polynomial
\[\tilde{y}(x,\theta)=a_2x^2+a_1x+a_0,\]
where
 \begin{eqnarray}\label{coefficients}
a_2&=&\frac{15}{32}(w^2_1 (w^1_1)^2 + w^2_2(w^1_2)^2), \nonumber\\
a_1&= &\frac{1}{2}(w^1_1w^2_1 + w^1_2w^2_2) + \frac{15}{16}(b^1_1w^1_1w^2_1 + b^1_2w^1_2w^2_2)\nonumber\\
a_0&=&b^2_1 + \frac{1}{32}w^2_1(15(b^1_1)^2 + 16b^1_1 + 3) + \frac{1}{32}w^2_2(15(b^1_2)^2 + 16b^1_2 + 3).
\end{eqnarray}
We consider a regression problem by using the ReLU neural network to fit  $y=|x|$ and have the following polynomial system
 \begin{eqnarray}\label{poly}
P(\theta)=
		\left(
		\begin{array}{cr}
\frac{15}{32}(w^2_1 (w^1_1)^2 + w^2_2(w^1_2)^2)-\frac{15}{16} \\
\frac{1}{2}(w^1_1w^2_1 + w^1_2w^2_2) + \frac{15}{16}(b^1_1w^1_1w^2_1 + b^1_2w^1_2w^2_2) \\
b^2_1 + \frac{1}{32}w^2_1(15(b^1_1)^2 + 16b^1_1 + 3) + \frac{1}{32}w^2_2(15(b^1_2)^2 + 16b^1_2 + 3)-\frac{3}{16} \\
(w^1_1)^2-(w^1_2)^2\\
(w^2_1)^2-(w^2_2)^2\\
(b_1^1)^2-(b^1_2)^2\\
b^{2}
		\end{array}
		\right).
\end{eqnarray}
By choosing a multi-homogenous start system \cite{SWbook}, the homotopy method is used to obtain six solutions after tracking $96$ solution paths, namely,
 \begin{eqnarray}\label{poly2}
\theta_1&=&(   4.7773,   -4.7773,  0.0438,    0.0438,   1.6228,    1.6228,     0)^T,\nonumber\\
\theta_2&=&(  1.0000,   -1.0000,       1.0000, 1.0000, 0.0000,   -0.0000, 0)^T,\nonumber\\
\theta_3&=&(    0.7588,   -0.7588,    1.7369,    1.7369,   -0.9801,   -0.9801,   0)^T,\nonumber\\
\theta_4&=&(  -1.0061,    1.0061,      0.9879,0.9879 -1.0690,   -1.0690,   0)^T,\nonumber\\
\theta_5&=&(   -1.4877,    1.4877,    0.4518,    0.4518,    0.1927,    0.1927,   0)^T,\nonumber\\
\theta_6&=&(     1.2318,   -1.2318,    0.6591,    0.6591,    0.0895,    0.0895,   0)^T.
\end{eqnarray}
Similarly, by solving the polynomial system with $P_4 ReLu$ approximation, we obtain five solutions after tracking $1280$ solution paths, namely,
 \begin{eqnarray}\label{poly4}
\theta_1&=&(    -1.3251,     0.3664,        0.8967,    -42.4003,  0.7037,    -0.7037,     0)^T,\nonumber\\
\theta_2&=&(  1.0000,   -1.0000,       1.0000, 1.0000, 0.0000,   -0.0000, 0)^T,\nonumber\\
\theta_3&=&(    -0.9418,    -0.2975,     1.9320,     61.3142,     0.4898,    -0.4898,   0)^T,\nonumber\\
\theta_4&=&(   -0.3664,     1.3251,    -42.4003,     0.8967,    -0.7037,     0.7037,   0)^T,\nonumber\\
\theta_5&=&(     -0.2975,    -0.9418,        61.3142,     1.9320, -0.4898,0.4898,   0)^T.
\end{eqnarray}
Then we plug the solutions into the ReLU neural network and plot the $y(x,\theta)$ in Fig. \ref{Fig:p2} which clear shows that we can recover $|x|$ by using the ReLU neural network via solving the polynomial system $P(\theta)$.

\begin{figure}[!ht]
\centering
        \includegraphics[width=1.5in]{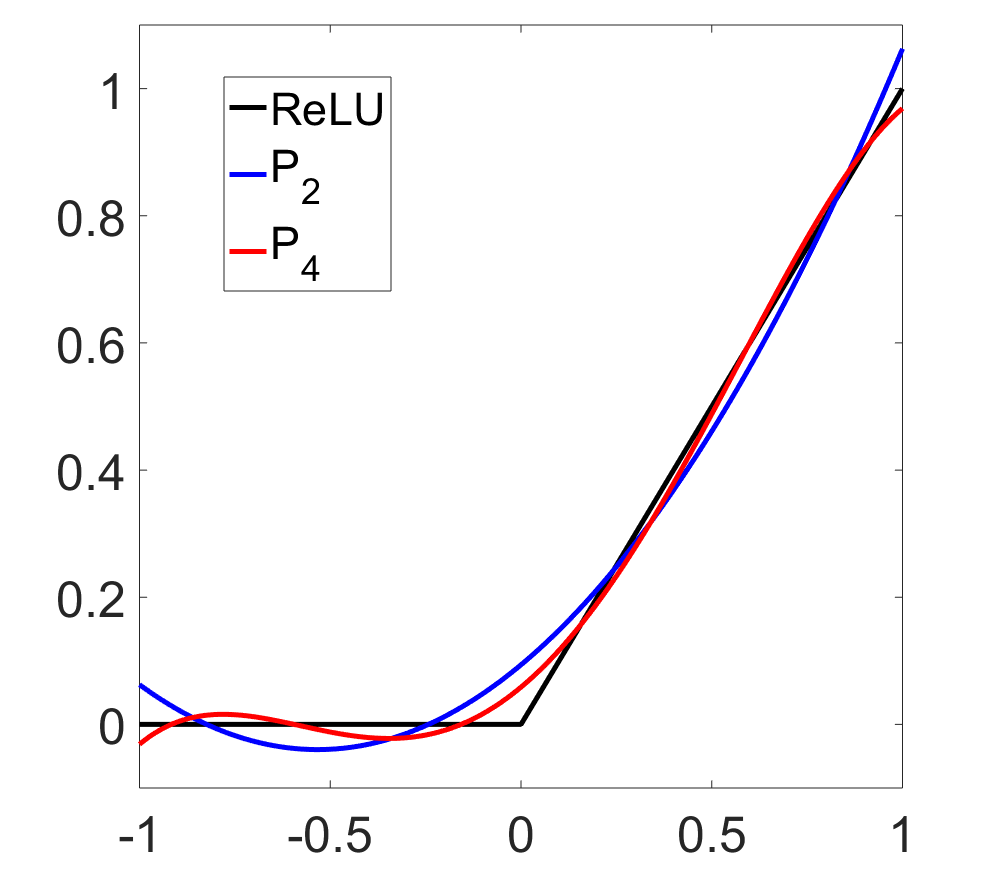}
  \includegraphics[width=1.5in]{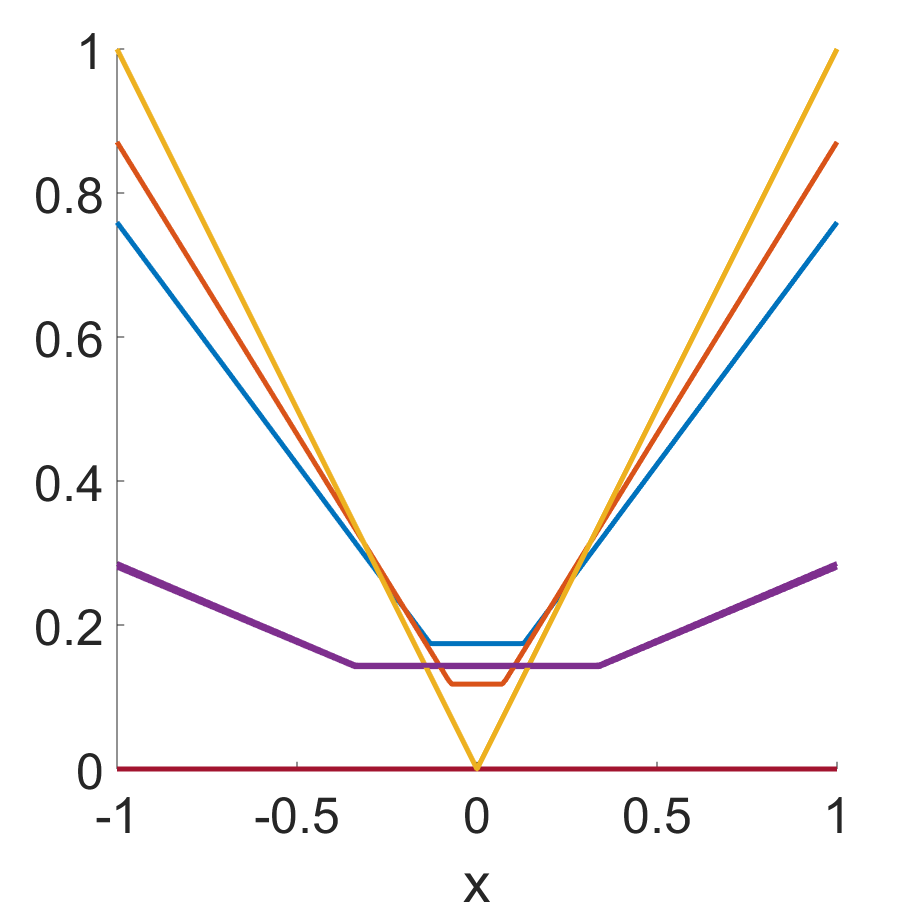}    \includegraphics[width=1.5in]{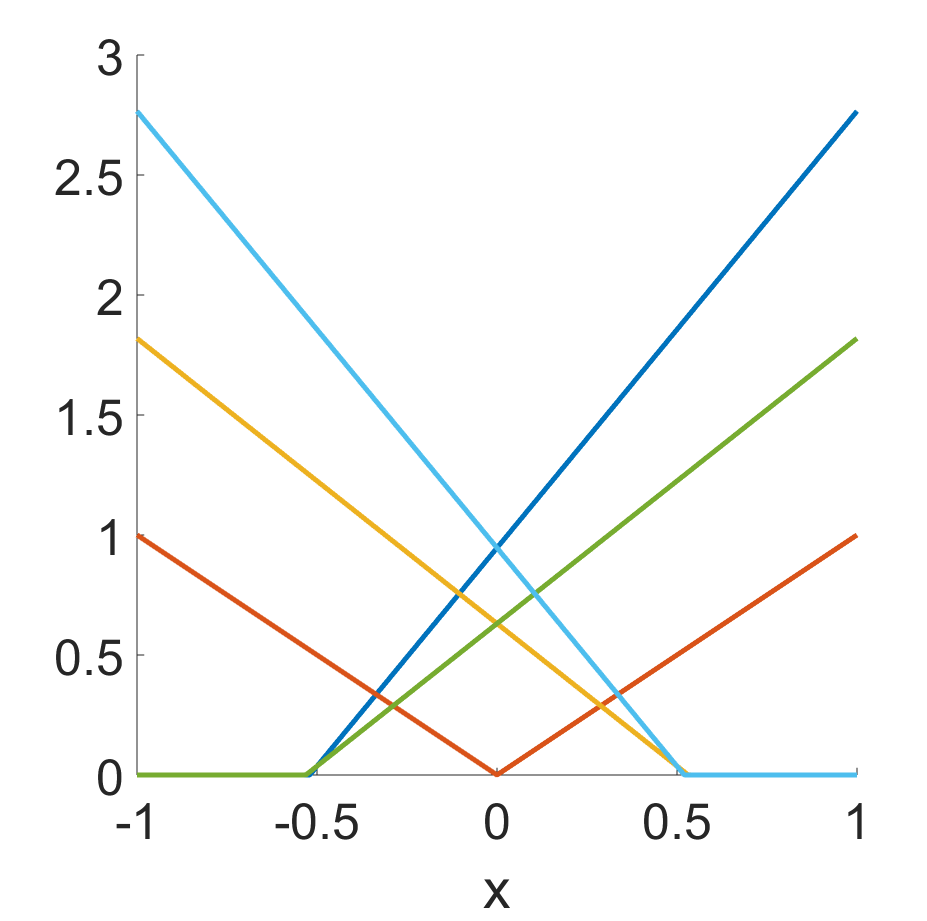}
\caption{{\bf Left:} A comparison between the ReLU function and approximated polynomials, $P_2$ and $P_4$; {\bf Middle and right:} Plots of the ReLU neural network, $y(x,\theta)$ defined in (\ref{1hidden}), with solutions in (\ref{poly2}) (middle) and (\ref{poly4}) (right) via the polynomial approximations. }\label{Fig:p2}
\end{figure}

\noindent{\bf Multi-hidden-layer neural network:} We employ $n$-hidden-layer ReLU networks with width 2 to approximate $f(x)=|x|$ with $P_2$ polynomial approximation.
The numerical results of different multi-hidden-layer neural networks are summarized in Table \ref{tab:ml} which clearly shows that the computational cost increase exceptionally as the layer goes deep. Therefore, the homotopy continuation method cannot be applied to solve large-scale neural networks directly.

\noindent{\bf Summary:} Instead of solving the optimization problem, we solve a system of polynomial equations directly, namely, $\tilde{y}(x_i)=|x_i|$, where $x_i$ is the sample points on $[-1,1]$. Therefore, the homotopy continuation method can be used to solve the neural networks after approximating the activation function by polynomials and provides good initialization for the original neural network but cannot be applied to deep neural networks due to the high computational cost.

\begin{table}
	\caption{Numerical results for different neural networks}\label{tab:ml}
	\centering
	\begin{tabular}{|c|c|c|}\hline
		\# of hidden layers &\# of solutions tracked &\# of real solutions\\\hline
		2& 865&10\\\hline
		3& 28,475&32\\\hline
		4&27,283,365 &129\\\hline
	\end{tabular}
\end{table}

\subsection{Proof of Theorem 4.2}\label{pr4.1}

{\bf Notations in the proof of Theorem 4.1}
\begin{itemize}
  \item {\bf Model without re-initialization:} We rewrite the neural network model as
\begin{equation}\label{model1}
f^\ell = W^\ell \sigma (f^{\ell-1}(x)) + b^\ell, \ell=1,\dots,n
\end{equation}
with $f^1 = W^1 x + b^1$ for all $x \in \Omega$. $\{ (W^\ell, b^\ell)\}_{\ell=1}^n$ are
initialized randomly such as
He's initialization~\cite{he2015delving} or the LPS initialization without initializations. We denote $W^{\ell}_{sk}$ for the element on $s$-th row and $k$-th column and  $b^\ell_{s}$ for the $s$-th element.

  \item {\bf Model with re-initialization:} The neural network with one re-initialization (or $N$ times re-initialization as in the proof of Theorem {\bf 4.2}) is denoted as
\begin{equation}\label{model2}
\bar f^\ell = \bar W^\ell \sigma ( \bar f^{\ell-1}(x)) + \bar b^\ell, \ell=1,\dots,n
\end{equation}
with $\bar f^1 = \bar W^1 x + \bar b^1$ for all $x \in \Omega$.
Here $\{ (\bar W^\ell, \bar b^\ell)\}_{\ell=1}^L$ are re-initialized with probability $p_\ell$ for these negative
elements.

  \item We consider these next sets related to the initialization of neural network model:
	\begin{align*}
	A_\ell &= \{\exists j \in \{1,\dots,\ell-1\} \text{ such that } \phi(\bar f^j( x)) = {0} \quad \forall x \in \Omega\},\\
	A_\ell^c &=
	\{\text{$\forall 1 \le j < \ell$ there exists $x \in \Omega$ such that } \phi(\bar f^{j}(x)) \ne {0}\},\\
	{A}_{\ell,x}^c &= \{\forall 1 \le j < \ell, \hspace{0.1cm} \phi(\bar f^{j}(x)) \ne {0} \}, \\
	{A}_{\ell,x} &= \{\exists \hspace{0.1cm} j \in \{1,\dots,\ell-1\} \text{ such that } \phi(\bar f^{j}(x)) = {0} \}.
	\end{align*}
	In addition, we define the following sets:
	\begin{align*}
	\bar D^{j-1}_{s,x} = \{ \bar {W}^{j-1}_s \phi(\bar f^{j-2}(x)) + \bar {b}^{j-1}_s \le 0 \}, \\
	 D^{j-1}_{s,x} = \{  {W}^{j-1}_s \phi( \bar f^{j-2}(x)) +  {b}^{j-1}_s \le 0 \},\\
	C_{sk} = \{ W^{j-1}_{sk} \text{ is chosen to be re-resampled} \},\\
	R_j = \left\{ \text{the $j$-th layer is chosen to be re-initialized} \right\},
	\end{align*}
	where $P(R_j) = p_j=\frac{2^j}{2^{n+1}-1}$ by {\bf Algorithm \ref{alg1}}.

  \item {\bf  The probability density function of the normal distribution}, namely $\mathcal N(0,1)$ is
\[\phi(x)=\frac{1}{\sqrt{2\pi}}e^{-1/2 x^2}.\]

{\bf The Gauss  error function:} is defined as	\begin{equation}
	{\rm erf}(x) = \frac{2}{\sqrt{\pi}} \int_{0}^x e^{-w^2} dw.
	\end{equation}

\end{itemize}

The notations and definitions in the proof are summarized in the Appendix.
  	Based on Lemma A.1 in \cite{lu2019dying}, we have
	\begin{equation}
	P(f^\ell(x) \text{ is born dead in  } \Omega) = P(A_\ell).
	\end{equation}
Because ${A}_{\ell,x}^c \subset A_\ell^c$ for all $x \in \Omega$, thus	\begin{equation}
	P(A_\ell) = 1-P(A_\ell^c) \le 1-P({A}_{\ell,x}^c).
	\end{equation}
	Moreover, since  $P({A}_{1,x}^c)=1$
	for any $x \neq 0$ in $\Omega$, we have the Bayes
	formula below:
	\begin{equation} \label{app:thm1:eqn3}
	P({A}_{\ell,x}^c)
	= P({A}_{\ell,x}^c | {A}_{\ell-1,x}^c)P({A}_{\ell-1,x}^c)
	= \dots
	= P({A}_{1,x}^c)\prod_{j=2}^\ell P({A}_{j,x}^c | {A}_{j-1,x}^c),
	\end{equation}
where \[P({A}_{j,x}^c | {A}_{j-1,x}^c)=1-P({A}_{j,x} |{A}_{j-1,x}^c).\]
Then we have
\begin{equation}\label{estimate}
\begin{aligned}
P(A_n) &=1 - \prod_{j=2}^{n} (1 - P({A}_{j,x} | {A}_{j-1,x}^c) ).
\end{aligned}
\end{equation}
Therefore, our proof mainly focus on estimating $P({A}_{j,x} | {A}_{j-1,x}^c)$, which has
the following decomposition:
		\begin{align*}
	&P({A}_{j,x} |{A}_{j-1,x}^c) \\
	%&P\left(\bar {W}^{j-1}_s \phi(\bar f^{j-2}(x)) + \bar {b}^{j-1}_s \le 0 | {A}_{j-1,x}^c\right) \\
%	=&P\left( \bar D^{j-1}_{s,x} | {A}_{j-1,x}^c\right) \\
	=&P\left( \left( {A}_{j,x} \text{ and } R^c_{j-1} \right) \textbf{ OR }
	\left(  {A}_{j,x} \text{ and } R_{j-1} \right) | {A}_{j-1,x}^c\right) \\
	=& P(R_{j-1}^c) \prod_{s=1}^{m_{j-1}}  P\left( \left(  \bar D^{j-1}_{s,x} | R_{j-1}^c \right) | {A}_{j-1,x}^c\right)
	+ P(R_{j-1}) \prod_{s=1}^{m_{j-1}} P\left( \left( \bar D^{j-1}_{s,x} | R_{j-1} \right)  | {A}_{j-1,x}^c\right)  \\
	\triangleq& (1-p_{j-1})\mathcal P_{1} + p_{j-1} \mathcal P_{2},
	\end{align*}
where $m_j$ the width of $j$-th layer.
	Since \begin{equation}
	P\left( \left(  \bar D^{j-1}_{s,x} | R_{j-1}^c \right) | {A}_{j-1,x}^c\right) = P\left(   D^{j-1}_{s,x}  | {A}_{j-1,x}^c\right) = \frac{1}{2},
	\end{equation}
 we obtain
	\begin{equation}
	\mathcal P_{1} = \prod_{s=1}^{m_{j-1}}  P\left( \left(  \bar D^{j-1}_{s,x} | R_{j-1}^c \right) | {A}_{j-1,x}^c\right)  = \big(\frac{1}{2}\big)^{m_{j-1}}.
	\end{equation}
	
For $\mathcal P_{2}$, by the definition of  ${A}_{j-1,x}^c$,
 there exists $1 \le k \le m_{j-2}$ such that $[\phi(\bar f^{j-2}(x))]_k > 0$.
	Thus, for any $1 \le s \le m_{j-1}$, we have
	\begin{equation}
	\begin{aligned}
	&P\left( \left( \bar D^{j-1}_{s,x} | R_{j-1} \right)  | {A}_{j-1,x}^c\right) \\
	= &\underbrace{P(W^{j-1}_{sk} >0 ) P\left( \big( \bar D^{j-1}_{s,x} | W^{j-1}_{sk} >0 | R_{j-1} \big)  | {A}_{j-1,x}^c\right)}_{\Pi_{1}}\\
	+&\underbrace{P(W^{j-1}_{sk} \le 0 ) P\left( \big( \bar D^{j-1}_{s,x} | W^{j-1}_{sk} \le 0 | R_{j-1} \big)  | {A}_{j-1,x}^c\right)}_{\Pi_2}.	\end{aligned}
	\end{equation}

\noindent{\bf Computations of $\Pi_1$:}
Since $P(W^{j-1}_{sk} >0 ) = \frac{1}{2}$, $\bar W^{j-1}_{sk} = W^{j-1}_{sk}$, and $W^{j-1}_{sk} > 0$,
	we have
	\begin{equation}\label{P1}
	\begin{aligned}
	&P\left( \big( \bar D^{j-1}_{s,x} | W^{j-1}_{sk} >0 | R_{j-1} \big)  | {A}_{j-1,x}^c\right) \\
	=&P \left( \left.\sum_{t\neq k} \bar W^{j-1}_{st} [\sigma(\bar f^{j-2}(x))]_t + \bar b^{j-1}_s \le -\bar W^{j-1}_{sk} [\sigma(\bar f^{j-2}(x))]_k \right| \bar W^{j-1}_{sk} >0 \right),
%	=&\frac{1}{2} - \delta^{j-1}.
	\end{aligned}
	\end{equation}
where $[\sigma(\bar f^{j-2}(x))]_t$ stands for the $t$-th element of the vector $\sigma(\bar f^{j-2}(x))$.
Because $\bar W^{j-1}_{st}$ and $\bar b^{j-1}_s$ are {\em independent and identically distributed}
with the same normal distribution $\mathcal N(0, v^2_{j-1})$ ($v_{j-1}=\frac{\sqrt{2}}{\sqrt{(m_{j-2}+1)m_{j-1}}}$ in {\bf Algorithm 1}), thus
	\begin{equation}
\mathcal X\triangleq \sum_{t\neq k} \bar W^{j-1}_{st} [\sigma(\bar f^{j-2}(x))]_t + \bar b^{j-1}_s \sim \mathcal N(0, \tilde{v}_{j-1}^2),
	\end{equation}
	where $\tilde{v}_{j-1}=v_{j-1}\sqrt{\sum_{t\neq k}[ \sigma(\bar f^{j-2}(x))]_t^2+1}$. Then (\ref{P1}) becomes
	\begin{equation}\label{eq:PI-}
	\begin{aligned}
	&P \left( \left.\mathcal X \le -\bar W^{j-1}_{sk} [\sigma(\bar f^{j-2}(x))]_k \right| \bar W^{j-1}_{sk} >0 \right) \\
=&\frac{P \left(\mathcal X \le -\bar W^{j-1}_{sk} [\sigma(\bar f^{j-2}(x))]_k\hbox{~and~} \bar W^{j-1}_{sk} >0 \right) }{P(\bar W^{j-1}_{sk} >0)} \\
	=&\frac{ \int_{0}^{\infty} \varphi\big(\frac{w}{v_{j-1}}\big)\int_{-\infty}^{-w [\sigma(\bar f^{j-2}(x))]_k} \varphi\big(\frac{y}{\tilde{v}_{j-1}}\big) d y dw }{\frac{1}{2}}\\
	=&2 \int_{0}^{\infty} \varphi\big(\frac{w}{v_{j-1}}\big)\int_{-\infty}^{-w [\sigma(\bar f^{j-2}(x))]_k} \varphi\big(\frac{y}{\tilde{v}_{j-1}}\big) d y dw\\
	=&2 \int_{0}^{\infty} \varphi\big(\frac{w}{v_{j-1}}\big) \left[ \frac{1}{2} - \frac{\tilde{v}_{j-1}}{\sqrt{\pi}}{\rm erf}\big(\frac{[\sigma(\bar f^{j-2}(x))]_k}{\sqrt{2}\tilde{v}_{j-1}} w\big)\right] dw \\
	=& \frac{1}{2} - 2\frac{\tilde{v}_{j-1}}{\sqrt{\pi}}\int_{0}^{\infty} \varphi\big(\frac{w}{v_{j-1}}\big) {\rm erf}\big(\frac{[\sigma(\bar f^{j-2}(x))]_k}{\sqrt{2}\tilde{v}_{j-1}} w\big) dw.
	\end{aligned}
	\end{equation}

By denoting
	\begin{equation}
	\delta_{j-1} = 2\frac{\tilde{v}_{j-1}}{\sqrt{\pi}}\int_{0}^{\infty} \varphi\big(\frac{w}{v_{j-1}}\big) {\rm erf}\big(\frac{[\sigma(\bar f^{j-2}(x))]_k}{\sqrt{2}\tilde{v}_{j-1}} w\big) dw\le \frac{1}{2},
	\end{equation}
we conclude that
	\begin{equation}
	\Pi_1 = P(W^{j-1}_{sk} >0 ) P\left( \left( \bar D^{j-1}_{s,x} | W^{j-1}_{sk} >0 | R_{j-1} \right)  | {A}_{j-1,x}^c\right) = \frac{1}{2}(\frac{1}{2} - \delta_{j-1}).
	\end{equation}
	
{\bf Computations of $\Pi_2$:}	By the definition of  $C_{sk}$,
 we have
	\begin{equation}
	P\left( (C_{sk} | W^{j-1}_{sk} \le 0 ) | R_{j-1} \right)  = \frac{1}{2}.
	\end{equation}
Therefore,
	\begin{equation}
	\begin{aligned}
	 &P\left( \left( \bar D^{j-1}_{s,x} | W^{j-1}_{sk} \le 0 | R_{j-1} \right)  | {A}_{j-1,x}^c\right)  \\
	 =&\underbrace{P\left( (C_{sk} | W^{j-1}_{sk} \le 0 ) | R_{j-1} \right)  P\left( \left( \bar D^{j-1}_{s,x} | C_{sk}  | W^{j-1}_{sk} \le 0 | R_{j-1} \right)  | {A}_{j-1,x}^c\right)}_{C_1}   \\
	 +&\underbrace{P\left( (C^c_{sk} | W^{j-1}_{sk} \le 0 ) | R_{j-1} \right)  P\left( \left( \bar D^{j-1}_{s,x} | C^c_{sk}  | W^{j-1}_{sk} \le 0 | R_{j-1} \right)  | {A}_{j-1,x}^c\right)}_{C_2}.
	\end{aligned}
	\end{equation}
Due to the independence, we have the following inequality:
	\begin{equation}
	P\left( \left( \bar D^{j-1}_{s,x} | C_{sk}  | W^{j-1}_{sk} \le 0 | R_{j-1} \right)  | {A}_{j-1,x}^c\right)  \le P\left( \left(  \bar D^{j-1}_{s,x} | R_{j-1}^c \right) | {A}_{j-1,x}^c\right) = \frac{1}{2}.
	\end{equation}
Because of the symmetry of the initialization and $[\phi(\bar f^{j-2}(x))]_k > 0$, we have $C_1 \leq \frac{1}{2}\times \frac{1}{2}$.
	
	For $C_2$, similarly to \eqref{eq:PI-}, we derive
	\begin{equation}
	\begin{aligned}
	&P\left( \left( \bar D^{j-1}_{s,x} | C^c_{sk}  | W^{j-1}_{sk} \le 0 | R_{j-1} \right)  | {A}_{j-1,x}^c\right) \\
	=&P \left( \left.\sum_{t\neq k} \bar W^{j-1}_{st} [\phi(\bar f^{j-2}(x))]_t + \bar b^{j-1}_s \le - W^{j-1}_{sk} [\phi(\bar f^{j-2}(x))]_k \right|  W^{j-1}_{sk} \le 0 \right)\\
	= &\frac{1}{2} + \delta_{j-1},
	\end{aligned}
	\end{equation}
which implies that $C_2 = \frac{1}{2}(\frac{1}{2} +\delta^{j-1})$.
	
Therefore
	\begin{equation}
	\begin{aligned}
	&P\left( \left( \bar D^{j-1}_{s,x} | R_{j-1} \right)  | {A}_{j-1,x}^c\right) \\
	= &\Pi_{1} + \Pi_2  =\Pi_{1} + P(W^{j-1}_{sk} \le 0 )(C_1 +C_2)\\
	\le&\frac{1}{2}(\frac{1}{2} - \delta_{j-1}) + \frac{1}{2}(\frac{1}{4} + \frac{1}{2}(\frac{1}{2}+\delta_{j-1})) \\
	=&\frac{1}{2}(1-\frac{\delta_{j-1}}{2}),
	\end{aligned}
	\end{equation}
	which leads to
	\begin{equation}
	\mathcal P_2 \le  \prod_{s=1}^{m_{j-1}} P\left( \left( \bar D^{j-1}_{s,x} | R_{j-1} \right)  | {A}_{j-1,x}^c\right) = 2^{-m_{j-1}}(1 -\frac{ \delta_{j-1}}{2})^{m_{j-1}}.
	\end{equation}
	
	Finally, we get
	\begin{equation}
	\begin{aligned}
	P({A}_{j,x} |{A}_{j-1,x}^c) &= 	(1-p_{j-1})\mathcal P_{1} + p_{j-1} \mathcal P_{2} \\
	&\le 2^{-m_{j-1}} \left((1-p_{j-1}) + p_{j-1}(1-\frac{\delta_{j-1}}{2})^{m_{j-1}}\right).
	\end{aligned}
	\end{equation}

In summary, we have the final estimate by (\ref{estimate})
 \begin{equation}
\begin{aligned}
P(A_n) &\le 1 - \prod_{\ell=1}^{n-1} \left(1 - 2^{-m_{\ell}} \left((1-p_{\ell}) + p_{\ell}(1-\frac{\delta_{\ell}}{2})^{m_{\ell}}\right) \right).
\end{aligned}
\end{equation}

\subsection{Proof of Theorem 4.2}\label{pr4.2}

{\bf Notations in the proof of Theorem 4.2}
We have the following notations
\begin{itemize}
  \item We denote	$\bar \theta^\ell = (\bar W^\ell, \bar b^\ell) \in \mathbb{R}^{m_\ell \times (m_{\ell-1} + 1)}$.
  \item We define the following set
  	\begin{equation}
	\begin{aligned}
	E_\ell &= \{  \bar \theta^\ell_{ij} >0 \text{ for all } 1\le i\le m_\ell, 1\le j \le m_{\ell-1}+1 \},\\
	R_\ell^{k} &= \{  \text{ the $\ell$-th layer is chosen to be re-initialized } k \text{ times}\}, \\
	C^{k,t}(x) &= \{ x \text{ is chosen to be re-sampled } t \text{ times in all } k \text{ times re-initializations} \}.
	\end{aligned}\nonumber
	\end{equation}
\end{itemize}

Since $\sigma(x) = \max\{0,x \} \ge 0$, we have
	\begin{equation}
		P\left( \bar {f}^n \text{ is born dead in } \Omega \right) \le 1- \prod_{\ell=1}^{n-1}P\left( E_\ell \right),
	\end{equation}
where \begin{equation}\label{ELL}
P(E_\ell) = \sum_{k=0}^N P(E_\ell | R_\ell^k ) P(R_\ell^k) \hbox{~and~}P(R_\ell^k) = \tbinom{N}{k} p_\ell^k (1-p_\ell)^{N-k}.
\end{equation}
Since \begin{equation}
\begin{aligned}
P(E_\ell | R_\ell^k ) &=  \prod_{ 1\le i\le m_\ell, 1\le j \le m_{\ell-1}+1} P(\bar \theta^\ell_{ij} > 0 | R_\ell^k) = \left(  1 - P(\bar \theta^\ell_{ij} \le  0 | R_\ell^k) \right)^{M_\ell},
\end{aligned}
\end{equation}
where $M_\ell = m_\ell \times (m_{\ell-1} + 1)$,
with $t$ times re-samplings in $k$ times re-initializations, for any  $1\le i\le m_\ell, 1\le j \le m_{\ell-1}+1$, we have
\begin{equation}\label{PER}
P(\bar \theta^\ell_{ij} \le  0 | R_\ell^k) = \sum_{t=0}^k P\left(\big(\bar \theta^\ell_{ij} \le  0 | C^{k,t}(\bar \theta^\ell_{ij}) \big)  | R_\ell^k\right) P(C^{k,t} (\bar \theta^\ell_{ij})  | R^\ell_k).
\end{equation}
For $P\left(\big(\bar \theta^\ell_{ij} \le  0 | C^{k,t}(\bar \theta^\ell_{ij}) \big)  | R_\ell^k\right)$, we know
that $\bar \theta^\ell_{ij} \le  0 | C^{k,t}(\bar \theta^\ell_{ij}) \big)  | R_\ell^k$ occurs if and only
if both the originally sampled $\theta^{\ell}_{ij}$ and all the $t$ times re-sampled $\bar \theta^{\ell}_{ij}$ are negative, namely,
\begin{equation}
\begin{aligned}
P\left(\big(\bar \theta^\ell_{ij} \le  0 | C^{k,t}(\bar \theta^\ell_{ij}) \big)  | R_\ell^k\right) = P (\theta^\ell_{ij} \le 0) P(\bar \theta^{\ell}_{ij} \le 0 \text{ after } t \text{ times re-sampling}) =\frac{1}{2} \left( \frac{1}{2}\right)^t.
\end{aligned}\nonumber
\end{equation}
Moreover, \begin{equation}
P(C^{k,t} (\bar \theta^\ell_{ij})  | R^\ell_k)= \tbinom{k}{t} \left(\frac{1}{2}\right)^{t}\left(\frac{1}{2}\right)^{k-t},
\end{equation}
(\ref{PER}) becomes
\begin{equation}
\begin{aligned}
P(\bar \theta^\ell_{ij} \le  0 | R_\ell^k) = \frac{1}{2} \sum_{t=0}^k   \tbinom{k}{t} \left(\frac{1}{4}\right)^{t}\left(\frac{1}{2}\right)^{k-t}= \frac{1}{2} \left( \frac{3}{4} \right)^k.
\end{aligned}
\end{equation}

Therefore, for $\forall~  0\le k \le N$, we have
\begin{equation}
\begin{aligned}
P(E_\ell | R_\ell^k ) &= \left(  1 - P(\bar \theta^\ell_{ij} \le  0 | R_\ell^k) \right)^{M_\ell} = \left(  1 - \frac{1}{2} \left( \frac{3}{4} \right)^k \right)^{M_\ell} \ge 1 - M_\ell \frac{1}{2} \left( \frac{3}{4} \right)^k.
\end{aligned}
\end{equation}
By (\ref{ELL}), we have  \begin{equation}
\begin{aligned}
P(E_\ell) \ge \sum_{k=0}^N \left(1 -  \frac{M_\ell}{2} \left( \frac{3}{4} \right)^k \right) \tbinom{N}{k} p_\ell^k (1-p_\ell)^{N-k} = 1 -  \frac{M_\ell}{2} \left( 1 -  \frac{p_\ell}{4}\right)^N.
\end{aligned}
\end{equation}
In summary, we conclude that, as $N \to \infty$,
\begin{equation}
\begin{aligned}
&P\left( \bar {f}^n \text{ is born dead in } \Omega \right) \le &1 - \prod_{\ell=1}^{n-1} \left( 1 -  \frac{M_\ell}{2} \left( 1 -  \frac{p_\ell}{4}\right)^N \right) \rightarrow 0.
\end{aligned}
\end{equation}

{ 
{\subsection{Experimental details}
\subsubsection{Fully Connected Neural Networks}
We compare the LPS initialization with He initialization on two fully connected neural networks to fit 1D and 2D functions. Fig \ref{fc_dying_prob}  shows the probability of born dying ReLU for different initialization strategies by using the hyperparameters shown in Table \ref{fcc} over 1000 initialization. The probability is computed by
\[P=\frac{\hbox{the number of initialization such that the variance of $y(x)\leq 10^{-10}$} }{1000}.\]	
By using the hyperparameters shown in Table \ref{fcc}, we train neural networks based on different initialization strategies to find ``non-collapse" cases. More specifically, we employ the Adam optimization algorithm \cite{kingma2014adam} with a learning rate of 0.001 with 4000
training steps. The training loss is based on the standard mean square error
(MSE) \[\mathcal{L}(\theta) =
\frac{1}{n}\sum_{i=1}^n\|(f(x_i;\theta)-y_i)\|_2^2.\]
For different function, $f_i$ ($i=1,\cdots,4$), we set different collapse thresholds to distinguish ``collapse" and  ``non-collapse" cases.
\begin{table}[!ht]
	\centering
	\begin{tabular}{|c|c|c|c|c|}
		\toprule
		Parameter & f1 & f2 & f3 & f4\\
		\midrule
		Number of sample points & 21 & 21 & 100 & 441\\
		Number of hidden layers & 10 & 10 & 10 & 20\\
		Size of hidden layers & 2 & 2 & 2 & 4\\
		Learning rate & $10^{-3}$ & $10^{-3}$ & $10^{-3}$ & $10^{-3}$\\
		Collapse threshold & 0.09 & 0.2 & 0.2 & 0.2\\\hline
		Optimizer &\multicolumn{4}{|c|}{ Adam} \\
%		Bias initialization &  True & True & False & False\\
		Dying threshold & \multicolumn{4}{|c|}{$10^{-10}$ }\\
		Training steps per run & \multicolumn{4}{|c|}{$4\times10^{3}$ }\\
		Number of runs & \multicolumn{4}{|c|}{$10^{3}$}\\
		\bottomrule
	\end{tabular}
\caption{Hyperparameters for the fully connected neural networks}\label{fcc}
\end{table}

\subsubsection{Convolutional neural networks}
We use the hyperparameters shown in Table \ref{LeNet} to train LeNet networks on the MNIST dataset \cite{MNIST}. We compute the mean and standard deviation of error rates and Good Local Minimum Percentage (GLMP) which refers the percentage of the validation accuracy greater than 99\% in 100 initialization. 
\begin{table}[h]
	\centering
	\begin{tabular}{|c|c|}
		\toprule
		Parameter & Value\\
		\midrule
		Number of epochs & 100\\
		Batch size & 64\\
		Initial learning rate & 0.05\\
		Learning rate schedule & Decrease by half every 30 epochs\\
		Weight decay & $5\times 10^{-4}$\\
		Optimizer & SGD with momentum = 0.9\\
		Bias initialization & False\\
		Number of runs & 100\\
		\bottomrule
	\end{tabular}
\caption{Hyperparameters for MNIST experiments}\label{LeNet}
\end{table}

All the hyperparameters shown in Table \ref{VGG} are used to train the VGG networks on the CIFAR-10 dataset \cite{CIFAR10}. 

\begin{table}[h]
	\centering
	\begin{tabular}{|c|c|}
		\toprule
		Parameter & Value\\
		\midrule
		Data Augmentation & \{RandomHorizontalFlip $\&$ RandomCrop\}\\
		Batch normalization & True \\
		Number of epochs & 250\\
		Batch size & 128\\
		Initial learning rate & 0.2\\
		Learning rate schedule & Decrease by half every 30 epochs\\
		Weight decay & $5\times 10^{-4}$\\
		Optimizer & SGD with momentum = 0.9\\
		Bias initialization & Both\\
		Number of runs & 10\\
		\bottomrule
	\end{tabular}
\caption{Hyperparameters for VGG networks on CIFAR-10.}\label{VGG}
\end{table}

All the deep residual network architecture considered in our experiment are reported in \cite{he2016deep}. We use the hyperparameters shown in Table \ref{Resnet}  train the ResNets on CIFAR-10, CIFAR-100 \cite{CIFAR10}, and ImageNet datasets \cite{deng2009imagenet}. 

\begin{table}[h]
	\centering
\scriptsize
	\begin{tabular}{|c|c|c|}
		\toprule
		Parameter & CIFAR-10 $\&$ CIFAR-100 & ImageNet\\
		\midrule
		Data Augmentation & \{RandomHorizontalFlip $\&$ RandomCrop\} & \{RandomHorizontalFlip $\&$ RandomResizedCrop\}\\
		Number of epochs & 250 & 90\\
		Batch size & 128 & \{ResNet-50: 128, ResNet-34: 256\}\\
		Initial learning rate & 0.2 & 0.1\\
		Learning rate schedule & Decrease by half every 30 epochs & Decrease by $1/10$ every 30 epochs\\
		Bias initialization & Both & False\\
		Number of runs & 10 & 1\\\hline
		Batch normalization & \multicolumn{2}{|c|}{True}\\
		Weight decay & \multicolumn{2}{|c|}{$5\times 10^{-4}$}\\
		Optimizer & \multicolumn{2}{|c|}{SGD with momentum = 0.9} \\
		\bottomrule
	\end{tabular}
\caption{Hyperparameters for the residual networks on CIFAR-10/CIFAR-100 and ImageNet datasets.}\label{Resnet}
\end{table}
}}

\bibliographystyle{RS}
\bibliography{LPS_init}

\end{document}